\documentclass[12pt]{article}
\usepackage[margin=1.25in]{geometry}
\usepackage{times}
\usepackage{amsmath, amsfonts, bm}
\usepackage{amsthm,graphicx}
\usepackage[dvipsnames]{xcolor}
\usepackage{mathtools}
\usepackage{amssymb}
\usepackage{pifont}

\usepackage{makecell}
\usepackage[utf8]{inputenc}
\usepackage[T1]{fontenc} 
\usepackage{amsfonts}  
\usepackage{mathrsfs} 
\usepackage{xspace}

\usepackage{booktabs}

\allowdisplaybreaks

\usepackage[unicode=true,
 bookmarks=false,
 breaklinks=false,pdfborder={0 0 1},colorlinks=false]
 {hyperref}
\hypersetup{
 colorlinks,citecolor=blue,filecolor=blue,linkcolor=blue,urlcolor=blue}

\usepackage{algorithm}
\usepackage{array}
\usepackage{makecell}
\usepackage{amssymb}
\usepackage{multirow}
\usepackage{color}
\usepackage[english]{babel}
\usepackage{graphicx}
\usepackage{natbib}
\usepackage{wrapfig}
\usepackage{epstopdf}
\usepackage{url}
\usepackage{graphicx}
\usepackage{color}
\usepackage{epstopdf}
\usepackage{algpseudocode}

\usepackage[T1]{fontenc}
\usepackage{bbm}
\usepackage{comment}
\usepackage{tikz}
\usepackage{tikzit}
\tikzstyle{Default}=[fill=white, draw=black, shape=circle]
\tikzstyle{Rec}=[fill=white, draw=black, shape=rectangle]

\tikzstyle{Unidirectional}=[->]
\tikzstyle{Double Arrow}=[<->]
\tikzstyle{Line}=[-]
\tikzstyle{Dashed arrow}=[->, dashed]
\tikzstyle{Dashed double arrow}=[<->, dashed]
\tikzstyle{dashed line}=[-, dashed]
\tikzstyle{blue line}=[-, color=blue]
\tikzstyle{Dotted Arrow}=[->, dotted]

\usepackage{booktabs,caption}
\usepackage[flushleft]{threeparttable}

\usepackage{tikz}
\usepackage{hyperref}
\hypersetup{colorlinks=true,citecolor=blue,linkcolor=red}

\usetikzlibrary{arrows}

\usepackage{amsmath}
\usepackage{amsthm}
\usepackage{subfig}
\usepackage{enumitem}

\newtheorem{theorem}{Theorem}
\newtheorem*{theorem*}{Theorem}
\newtheorem{lemma}{Lemma}
\newtheorem{definition}{Definition}

\newtheorem{corollary}{Corollary}

\newtheorem{remark}{Remark}

\newtheorem{assumption}{Assumption}

\usepackage{mdframed}
\mdtheorem{problem}{Problem}

\newcommand{\eps}{\epsilon}
\newcommand{\B}{\mathcal{B}}
\newcommand{\A}{\mathcal{A}}
\newcommand{\N}{\mathcal{N}}
\newcommand{\BN}{\mathbb{N}}
\newcommand{\R}{\mathbb{R}}
\newcommand{\Z}{\mathcal{Z}}
\newcommand{\VC}{\mathcal{V}}
\newcommand{\E}{\mathbb{E}}

\newcommand{\diag}{\mathsf{diag}}

\newcommand{\BP}{\mathbb{P}}
\newcommand{\MP}{\mathcal{P}}
\newcommand{\doi}{\mathrm{do}}
\newcommand{\Q}{\mathcal{Q}}
\newcommand{\UC}{\mathcal{U}}
\newcommand{\UA}{\mathcal{U}_{A}}
\newcommand{\UAh}{\mathcal{U}_{A,h}}
\newcommand{\UAhp}{\mathcal{U}_{A,h+1}}
\newcommand{\dps}{d_{\mathrm{PSR}}}
\newcommand{\od}{o_{\mathrm{dummy}}}
\newcommand{\unif}{\mathrm{Unif}}
\newcommand{\dhi}{d_{\mathrm{PSR},h}}
\newcommand{\dhj}{d_{\mathrm{PSR},h-1}}
\newcommand{\dhf}{d_{\mathrm{PSR},j_1-1;f}}
\newcommand{\eo}{\epsilon_{\mathrm{op}}}
\newcommand{\Fil}{\mathfrak{F}}
\newcommand{\eb}{\epsilon_{\mathrm{b}}}
\newcommand{\ncov}{\mathcal{N}_{\mathcal{F}}(\epsilon_{\mathrm{b}})}
\newcommand{\GC}{\mathcal{G}}
\newcommand{\GCU}{\mathcal{G}_{\mathrm{u}}}

\newcommand{\MO}{\mathbb{O}}
\newcommand{\MT}{\mathbb{T}}
\newcommand{\dpsh}{d_{\mathrm{PSR},h}}
\newcommand{\boo}{\boldsymbol{o}}
\newcommand{\BOO}{\boldsymbol{O}}

\newcommand{\D}{\mathcal{D}}
\newcommand{\FC}{\mathcal{F}}

\newcommand{\T}{\mathcal{T}}
\newcommand{\BO}{\mathcal{O}}
\newcommand{\TO}{\widetilde{\mathcal{O}}}
\newcommand{\SC}{\mathcal{S}}

\newcommand{\reg}{\mathrm{Regret}}

\newcommand{\ba}{\boldsymbol{a}}

\newcommand{\proj}{\mathrm{Proj}}
\newcommand{\col}{\mathrm{Col}}
\newcommand{\hpi}{\widehat{\pi}}

\newcommand{\poly}{\mathrm{poly}}

\newcommand{\bphi}{\overline{\phi}}
\newcommand{\bpsi}{\overline{\psi}}
\newcommand{\hphi}{\widehat{\phi}}
\newcommand{\bs}{\boldsymbol{s}}
\newcommand{\BD}{\overline{D}}
\newcommand{\ncove}{\mathcal{N}_{\mathcal{F}}(\epsilon)}
\newcommand{\BS}{\boldsymbol{S}}
\newcommand{\dtran}{d_{\mathrm{trans}}}
\newcommand{\dlin}{d_{\mathrm{lin}}}
\newcommand{\nphi}{\mathcal{Y}_{\Phi}}
\newcommand{\npsi}{\mathcal{Y}_{\Psi}}
\newcommand{\nmo}{\mathcal{Y}_{\MO}}
\newcommand{\epsl}{\eps_{\mathrm{LR}}}
\newcommand{\nmu}{\mathcal{Y}_{\mu}}
\newcommand{\phid}{\phi_{\mathrm{dec},h}}
\newcommand{\phim}{\phi_{\mathrm{dec},h+m-1}}
\newcommand{\bon}{\boldsymbol{1}}
\newcommand{\MC}{\mathcal{M}}
\newcommand{\OS}{\overline{\mathcal{S}}}
\newcommand{\os}{\overline{s}}
\newcommand{\ot}{\overline{t}}
\newcommand{\mainalg}{\texttt{CRANE}\xspace}

\definecolor{yxc}{RGB}{255,0,0}
\definecolor{yjc}{RGB}{190,0,255}
\definecolor{whz}{RGB}{0,155,0}

\usepackage{enumitem}
\setlist[itemize]{leftmargin=*}
\setlist[enumerate]{leftmargin=*}

\ifdefined\usebigfont

\usepackage{times}
\usepackage[fontsize=13pt]{scrextend}
\AtBeginDocument{\newgeometry{letterpaper,left=1.56in,right=1.56in,top=1.71in,bottom=1.77in}}

\else
\fi

\begin{document}

\title{PAC Reinforcement Learning for Predictive State Representations}
\author{%
	Wenhao Zhan\thanks{Princeton University. Email: \texttt{wenhao.zhan@princeton.edu}} \\
	\and
	Masatoshi Uehara\thanks{Cornell University. Email: \texttt{mu223@cornell.edu}} 
	\and
	Wen Sun\thanks{Cornell University. Email: \texttt{ws455@cornell.edu}}\\
	\and
	Jason D. Lee\thanks{Princeton University. Email: \texttt{jasonlee@princeton.edu}}\\
}
\date{\today}
\maketitle
%!TEX root = Optidice_arxiv.tex

\begin{abstract}
In this paper we study online Reinforcement Learning (RL) in partially observable dynamical systems. We focus on the Predictive State Representations (PSRs) model, which is an expressive model that captures other well-known models such as Partially Observable Markov Decision Processes (POMDP). PSR represents the states using a set of predictions of future observations and is defined entirely using observable quantities. We develop a novel model-based algorithm for PSRs that can learn a near optimal policy in sample complexity scaling polynomially with respect to all the relevant parameters of the systems. Our algorithm naturally works with function approximation to extend to systems with potentially large state and observation spaces. We show that given a realizable model class, the sample complexity of learning the near optimal policy only scales polynomially with respect to the statistical complexity of the model class, without any explicit polynomial dependence on the size of the state and observation spaces. Notably, our work is the first work that shows polynomial sample complexities to compete with the globally optimal policy in PSRs. Finally, we demonstrate how our general theorem can be directly used to derive sample complexity bounds for special models including  $m$-step weakly revealing and $m$-step decodable tabular POMDPs, POMDPs with low-rank latent transition, and POMDPs with linear emission and latent transition. 

 %where the successful probability of any possible future tests is a linear combinations of some given core tests. In order to deal with potentially large state and observation space, we develop a new algorithm called \mainalg, which can guarantee polynomial sample complexity when learning a near-optimal policy for PSRs in the context of general function approximation. We further show that tabular POMDPs, low-rank POMDPs and linear POMDPs can be viewed as special cases of PSRs and \mainalg is capable of achieveing efficient learning in all these models, which illustrates the generality of our model and algorithm.

%In this paper we study online policy learning in general sequential decision making processes. We focus on a model called predictive state representations (PSRs) where the successful probability of any possible future tests is a linear combinations of some given core tests. In order to deal with potentially large state and observation space, we develop a new algorithm called \mainalg, which can guarantee polynomial sample complexity when learning a near-optimal policy for PSRs in the context of general function approximation. We further show that tabular POMDPs, low-rank POMDPs and linear POMDPs can be viewed as special cases of PSRs and \mainalg is capable of achieveing efficient learning in all these models, which illustrates the generality of our model and algorithm.
\end{abstract}
%!TEX root = PSR_arxiv.tex

\section{Introduction}
%%%Partial observability and large state spaces pose a major hurdle to applying Reinforcement Learning (RL) algorithms. Existing work in RL has focused on scaling to large state spaces via function approximation in fully observable Markov decision processes (MDPs)~\citep{du2021bilinear,jiang2017contextual}. However, there is little work in studying both partial observability and large state spaces simultaneously~\citep{wang2022embed,cai2022sample}.

Efficient exploration strategies in reinforcement learning have been well investigated on many models from tabular models \citep{jaksch2010near,azar2017minimax} to models that enable us to use general function approximation~\citep{du2021bilinear,jiang2017contextual,jin2021bellman,foster2021statistical,sun2019model}. These works have focused on fully observable Markov decision processes (MDPs); however, their algorithms do not result in statistically efficient algorithms in partially observable Markov decision processes (POMDPs). Since the markovian properties of dynamics are often questionable in practice, POMDPs are known to be useful models that capture environments in real life. While strategic exploration in POMDPs was less investigated due to its difficulty, it has been actively studied in recent few years \citep{guo2016pac,azizzadenesheli2016reinforcement,jin2020sample}. In our work, beyond POMDPs, we consider Predictive state representation (PSR)~\citep{littman2001predictive,singh2004predictive,jaeger1998discrete} that is a more general model of controlled dynamical systems than POMDPs. 

PSRs are specified by the probability of a sequence of future observations/actions (referred to as a test) conditioned on the past history. Unlike the POMDP model, PSR directly predicts the future given the past \emph{without} modeling the latent state/dynamics. PSRs can model every POMDP, but potentially result in much more compact representations; there are dynamical systems that have finite PSR ranks, but that cannot be modeled by any POMDPs with finite latent states \citep{littman2001predictive,jaeger1998discrete}. 
%there are PSRs that require exponentially many latent variables in the size of the PSR~\citep{littman2001predictive}. 

PSRs are not only general but also amenable to learning and scalable. First, PSRs can be efficiently learned from exploratory data using a spectral learning algorithm \citep{boots2011closing} motivated by method-of-moments \citep{hsu2012spectral}. This learning algorithm allows us to perform fast closed-form sequential filtering, unlike EM-type algorithms that would be the most natural algorithm derived from POMDP perspectives. Secondly, while original PSRs are defined in the tabular setting, PSRs also support rich functional forms through kernel mean embedding \citep{boots2013hilbert}. Variants of PSRs equipped with neural networks have been proposed as well \citep{sun2016learning,downey2017predictive,venkatraman2017predictive,zhang2021reinforcement}.

%In the large state space setting, function approximation is used to either approximate the transition or dynamics for MDP/POMDP. 
In spite of the abovementioned advances in research on PSRs made in the recent two decades, strategic exploration without exploratory data has been barely investigated. To make PSRs more practical, it is of significant importance to understand how to perform efficient strategic exploration. To the best of the author's knowledge, \citet{jiang2016contextual,uehara2022provably} tackle this challenge; however, they fail to show results with polynomial sample complexity to compete with the globally optimal policy. Our aim is to obtain algorithms with polynomial sample complexity. Another desideratum for algorithms is to permit for general function approximation. This desideratum is important to enjoy the scalable property of PSRs.  
%%%Motivated by this, we propose using function approximation to compactly represent the parameters of the PSR.  
In summary, the key question we wish to address in this work is:
\begin{center}
	Can we design provably efficient RL algorithms for learning PSR with function approximation?
\end{center}

\begin{table}[!th]

\begin{tabular}{c|c|c|c}
\toprule
		& \begin{tabular}{c} $m$-step \\ weakly-revealing POMDPs  \end{tabular} & \begin{tabular}{c} $m$-step \\ decodable POMDPs \end{tabular} & PSRs \\  \midrule
\citet{efroni2022provable} &  & $+$  &  \\ 
\citet{liu2022partially}		 &  $+$ &  & \\ 
\citet{jiang2016contextual} &  &  & $\circ$ \\
\citet{uehara2022provably}  &  $\circ$  & $+$   & $\circ$ \\ 
Our Work	&   $+$    & $+$ & $+$  \\ \bottomrule
\end{tabular}

\caption{Comparsion of our work with existing works. $+$ means that algorithms can learn the near \emph{globally} optimal policy with polynomial sample complexities. Our work is the \emph{only} work that has a desirable guarantee on three models. 
In $m$-step weak revealing POMDPs, $\circ$ in \citet{uehara2022provably} means the sample complexity is quasi-polynomial but not polynomial. In $m$-step decodable POMDPs, all of the works have certain caveats.
 More specifically, in \citet{efroni2022provable,uehara2022provably}, it is unclear whether they can avoid $\mathrm{poly}(|\mathcal{O}|^m)$. On the other hand, our result can surprisingly avoid $\mathrm{poly}(|\mathcal{O}|^m)$ while we need a regularity assumption. For more details, refer to Section~\ref{sec:example}. In PSRs, $\circ$ in \citet{jiang2016contextual} means the guarantee is limited to reactive PSRs where the optimal value function depends on current observations. Similarly, $\circ$ in \citet{uehara2022provably} means the algorithm can compete with short-memory policies but not near globally optimal policies.} %\masa{There is also some subtlety of assumptions in $m$-step weak revealing POMDPs. Might be better to metion. But at the same time, I feel this might be too detailed stuff at this point. }}\whz{\checkmark.}
\end{table}

\paragraph{Contributions.} Our main contributions are summarized below.
\begin{enumerate}
\item We develop the first PAC learning algorithm for PSRs   that can compete with the globally optimal policy and identify the PSR rank $\dps$ as the key structural quantity of PSR systems. Starting with a realizable model class, our algorithm learns a near-optimal policy with sample complexity scaling polynomially in $\dps$ and the statistical complexity (log bracket number), without any explicit polynomial dependence on the size of state and observation space, as shown in the informal theorem below. Thus, our approach can be applied to large-scale partially observable systems. The following statement summarizes our main result.
\begin{theorem*}[\textit{Informal}]
For any $\eps>0$, we can learn an $\eps$-optimal policy with high probability when the number of samples $T$ is polynomial in 
\begin{align*}
	T=1/\eps^2\times\poly(\dps,|\UA|,1/\alpha,\log\ncov,H,|\A|,\log|\BO|),
\end{align*} 
where $\dps$ is the rank of the PSR, $|\UA|$ is the number of different action sequences in the core test set, $\alpha$ is the regularity parameter of the PSR, $\log\ncov$ is the log bracket number of the function class, $H$ is the horizon, $\A$ is the action space, and $\BO$ is the observation space.
\end{theorem*}
 
\item We demonstrate how our general result can be seamlessly applied to existing POMDP models with function approximation. These models include tabular $m$-step weak-revealing POMDPs \citep{liu2022partially} and tabular $m$-step decodable POMDPs \citep{efroni2022provable}. Especially, our work is the first work that ensures PAC guarantees with polynomial sample complexities for $m$-step weak-revealing POMDPs and $m$-step decodable POMDPs \emph{simultaneously}. We further show sample complexity results when these two types of POMDPs have additional two types of structures to permit for large state/observation space: with low-rank latent transition and with linear latent transition and observation distributions, which all have low $\dps$ much smaller than $\mathcal{|S|}$.
\end{enumerate}

 %%%%\jnote{add informal theorem here}

\paragraph{Notations.} In this work we use $[n]$ to denote the set $\{1,2,\cdots,n\}$ and $[n]^+$ to denote the set $\{0,1,2,\cdots,n\}$ for any positive integer $n$. For any set $\mathcal{C}$, we use $|\mathcal{C}|$ to denote its cardinality and $[x_c]_{c\in\mathcal{C}}$ to denote the vector whose entry is $x_c$ for all $c\in\mathcal{C}$. We also use $\Delta_{\mathcal{C}}$ to represent the set of all probability distributions over $\mathcal{C}$. For any vector $x$, we use $\Vert x\Vert_{1}$,
$\Vert x\Vert_{2}$ and $\Vert x\Vert_{\infty}$ to denote its $\ell_1$, $\ell_2$ and $\ell_{\infty}$ norm. For any matrix $M$, we use $(M)_{i,j}$ to denote the $(i,j)$-th entry of $M$ and $M^{\dagger}$ to denote the pseudo inverse of $M$. We also use $\Vert M\Vert_{\infty,\infty}$ to denote $\max_{i,j}|(M)_{i,j}|$ and $\Vert M\Vert_{1\mapsto1}$ to denote its $\ell_1$ norm $\sup_{\Vert x\Vert_1=1}\Vert Mx\Vert_1$. In addition, we use $\sigma_{\min}(M)$ to denote the minimum nonzero singular value of $M$ and $\sigma_{n}(M)$ to denote the $n$-the largest singular value of $M$.
\paragraph{Organization. } In Section \ref{sec:preli} we introduce the definition and key properties of PSRs and state the learning objective. In Section~\ref{sec:alg} we propose a new algorithm, \mainalg, to tackle the online PSR learning tasks. We characterize the sample complexity guarantee of \mainalg to learn a near-optimal policy for PSRs in Section~\ref{sec:main} and provide a proof sketch in Section~\ref{sec:sketch}. In Section~\ref{sec:example} we further illustrate the sample complexity of \mainalg to learn tabular PSRs and several POMDPs and compare the results with existing algorithms.

%%%%%\subsection{Related Work}
%%%%\input{related}

%\paragraph{Generalization and function approximation of RL in MDPs. } In Markovian environments, there is a growing literature that gives PAC bounds with function approximation under certain models. Some of the representative models are linear MDPs \citep{jin2020provably,yang2020reinforcement}, block MDPs \citep{du2019provably,misra2020kinematic,zhang2022efficient}, and low-rank MDPs \citep{agarwal2020flambe,uehara2021representation}.  
%Several general frameworks in  \citep{jiang2017contextual,sun2019model,jin2021bellman,foster2021statistical,du2021bilinear} characterize sufficient conditions for provably efficient RL. Each above  model is captured in these frameworks as a special case. 
%While our work builds on the bilinear/Bellman rank framework \citep{du2021bilinear,jiang2017contextual}, when we na\"ively reduce POMDPs to MDPs, the bilinear/Bellman rank is $\Theta(A^H)$. These two frameworks are only shown applicable to reactive POMDPs where the optimal policy only depends on the latest observation. However, this assumption makes the POMDP model very restricted. 

%and the Bellman rank framework \citep{jiang2017contextual}. 
%Particularly, we adopt the actor-critic framework that was originally proposed by \cite{jiang2017contextual} to partially observable systems.   \masa{i deleted }
%\looseness=-1 \ 
% since almost all existing models in Markovian environments are characterized as low bilinear rank models.  

\subsection{Related Work }

\paragraph{PSRs and its learning algorithm}

PSRs represent states as a vector of predictions about future events \citep{littman2001predictive,singh2004predictive,rosencrantz2004learning,hamilton2014efficient,thon2015links,grinberg2018learning}. Importantly, compared to well-known models of dynamical systems like HMMs that postulate latent state variables that are never observed, we do not need to refer to latent state variables and every definition relies on observable quantities. While PSRs were originally introduced in the tabular setting, PSRs can be extended to the non-tabular setting using conditional mean embeddings \citep{boots2013hilbert}. Using data obtained by exploratory open-loop policies such as uniform policies, \citet{boots2011closing,boots2013hilbert,zhang2021reinforcement} proposed a learning algorithm for dynamics by leveraging spectral learning \citep{kulesza2014low,hsu2012spectral,jiang2018completing}. Later, \citet{hefny2015supervised} pointed out an insightful connection between spectral learning and supervised learning (more specifically, instrumental variable regression when histories are instrumental). Based on this viewpoint, \citet{hefny2015supervised} proposed a two-stage regression learning algorithm. Compared to these settings, our setting is significantly challenging. This is because their goal is learning system dynamics with exploratory offline data while we want to learn the optimal policy when we don't have access to such exploratory data. 
%%%This difference in settings leads to a difference in primary assumptions, i.e., we only need knowledge of core tests, but they require not only knowledge of core tests but also core histories. 

%%%%%We remark abovementioned learning algorithms restrict their attention to linear PSRs like our paper. PSRs can be generalized to non-linear PSRs. For example, \citep{sun2016learning} tackle learning on nonlinear PSRs. We leave the extension to future works.  \masa{add in the last section?}

\paragraph{Provably efficient RL for POMDPs and PSRs. }
%%%%Partial observability is known to be a major challenge in RL in practice.  
Seminal works \citep{kearns1999approximate,even2005reinforcement} obtained $A^H$-type sample complexity bounds for POMDPs. We can avoid exponential dependence with more structural assumptions. Recently, there is a growing body of literature that discusses provably efficient RL in the online setting under various structures. 
 
 In the tabular setting, one of the most standard structural assumptions is an observability (i.e., weakly-revealing) assumption, which implies that observations retain information about latent states. Under observability and various additional assumptions, in \citet{azizzadenesheli2016reinforcement,guo2016pac,kwon2021rl}, favorable polynomial sample complexities are obtained by leveraging the spectral learning technique \citep{hsu2012spectral}. Later, \citet{jin2020sample,liu2022partially} improve these results and obtain polynomial sample complexity results under only observability assumptions. \citet{golowich2022learning,golowich2022planning} develop algorithms with quasi-polynomial sample and computational complexity under observability properties.

 %Among these three works, the former two rely on strong assumptions on the rank of the transition matrices and ergodicity to alleviate the difficulty of strategic exploration, while the third \citep{jin2020sample} operates under the most general setting and performs strategic exploration
%(see section 1.1 in \cite{jin2020sample} for an excellent summary). 
%we also work on the setting where one needs to perform strategic exploration.

In the non-tabular POMDP setting, several positive results are obtained. One of the most investigated models is linear quadratic gaussian (LQG), which is a partial observable version of LQRs. \citet{lale2020regret,simchowitz2020improper} proposed sub-linear regret algorithms.  %by using the Youla parameterization on the quantities called Nature's $y$. 
%These works use random policies for exploration, which is sufficient for LQG. Since random exploration strategy is not enough for tabular POMDPs, it is unclear if the existing techniques from LQG can be applied to solve general POMDPs.
 Polynomial sample complexities are obtained on other various POMDP models such as M-step decodable POMDPs \citep{efroni2022provable} where we can decode the latent state by $m$-step back histories (when $m=1$, it is Block MDP), weakly-revealing linear-mixture type POMDPs \citep{cai2022sample} where emission and transition are modeled by linear mixture models, weakly-revealing low-rank POMDPs \citep{uehara2022provably} where latent transition have low-rank structures. 
 % that is a generalization of observable LQG and tabular POMDPs
 Our proposed algorithm can capture all of the abovementioned models except for LQG. 

There are few works that discuss strategic exploration in PSRs. None of them obtain polynomial sample complexity results for learning approximate globally optimal policies~\citep{jiang2016contextual,uehara2022provably}. For details, refer to Section \ref{sec:main}. 
\section{Preliminaries}
In this section, we introduce the definition and key properties of PSRs. After that, we state our learning objective for PSRs with function approximation. 
\label{sec:preli}
\subsection{Predictive State Representations}
We consider an episodic sequential decision making process $\MP=\{\BO,\A,\BP,\{r_{h}\}_{h=1}^{H},H\}$, where $\BO$ is the observation space, $\A$ is the action space, $\BP$ is the system dynamics, $r_h$ is the reward function at $h$-th step and $H$ is the length of each episode. We suppose the reward $r_h$ at $h$-th step is a deterministic function of $(o_h,a_h)$ conditioned on the history $\tau_{h}$ where $\tau_{h}=(o_1,a_1,\cdots,o_h,a_h)$. 

We assume the initial observation $o_1$ of each episode follows a fixed distribution $\mu_1 \in \Delta_{\mathcal{O}}$. At step $h\in[H]$, the agent observes the observation $o_h$ and takes action $a_{h}$ based on the whole history $(\tau_{h-1},o_h)$. After that, the agent receives its reward $r_{h}(o_h,a_h)$ and the environment generates $o_{h+1}\sim\BP(\cdot|\tau_h)$. After the agent takes $a_H$, we suppose the environment will only generate dummy observations $\od$ no matter what actions the agent takes afterwards.

\paragraph{Policy and value.} A policy $\pi=\{\pi_h:(\BO\times\A)^{h-1}\times\BO\to\Delta_{\A}\}_{h=1}^H$ specifies the action selection probability at each step conditioned on the history $(\tau_{h-1},o_h)$. Given any policy $\pi$, its value $V^{\pi}$ characterizes the expected cumulative rewards as defined below:
\begin{align*}
	V^{\pi}:=\E_{\pi}\bigg[\sum_{h=1}^Hr_h\bigg],
\end{align*}
where the expectation is w.r.t. to the distribution of the trajectory induced by executing $\pi$ in the environment. We also use $\BP^{\pi}(\tau)$ to represent the probability of trajectory $\tau$ when executing policy $\pi$ in the environment.

\begin{figure}[t]
	\centering
	\includegraphics[width=\linewidth]{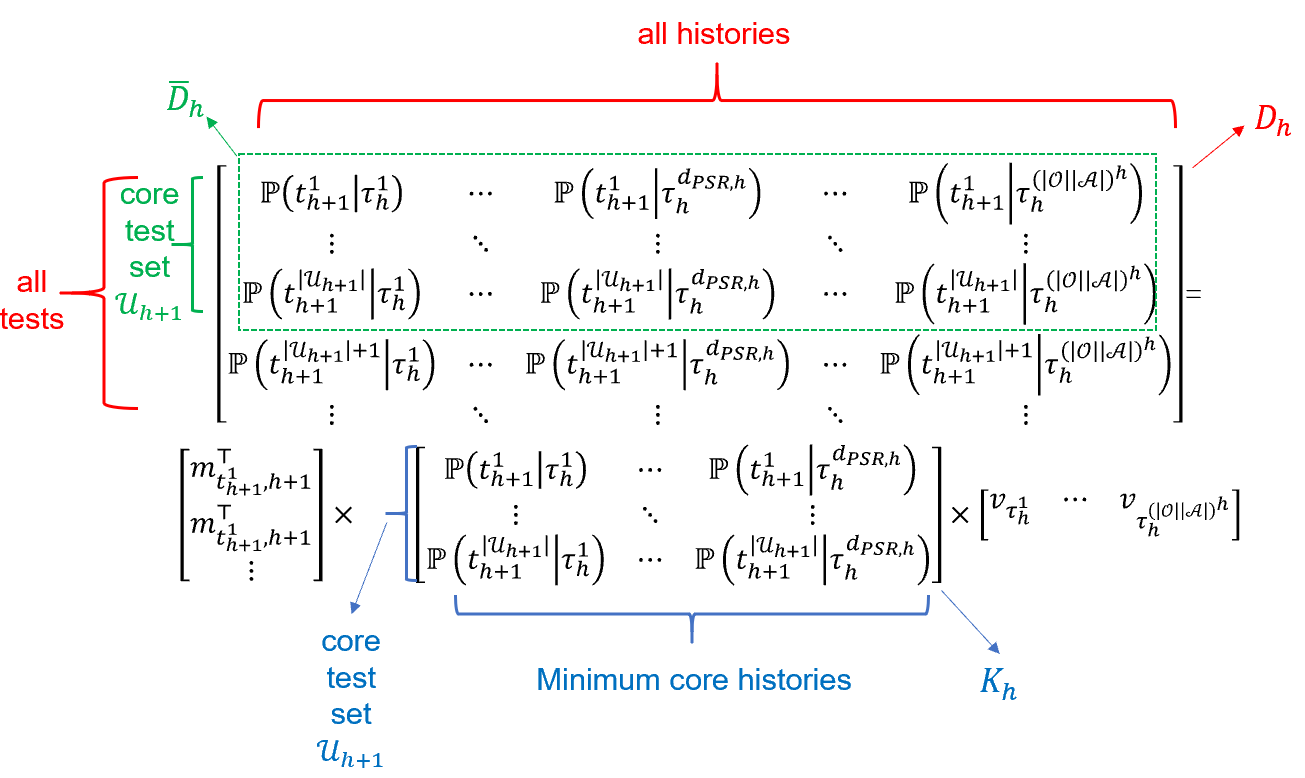}
	\caption{Illustration of the key concepts in PSRs using the system dynamics matrix  $D_h$ indexed by all tests and all histories. Denote $\dpsh$ as the rank of $D_h$. A core test set $\UC_{h+1}$ is a subset of tests such that the submatrix $\BD_h$ whose rows are indexed by tests in $\UC_{h+1}$ has rank $\dpsh$. Thus any row in $D_h$ can be written as a linear combination of the rows in $\BD_h$. A core test set $\UC_{h+1}$ whose size is exactly equal to $\dpsh$ is called a minimum core test set. The minimum core history set is a size--$\dpsh$ subset of histories such that the submatrix $K_h$ of $\BD_h$ whose columns are indexed by the history in the minimum core history set has rank $\dpsh$. Any column in $D_h$ can be written as a linear combination of columns indexed by histories in the minimum core history set.}
	\label{fig:core matrix}
\end{figure}

\subsubsection{Key Concepts in PSRs}

\paragraph{Tests and Linear PSRs.} A test is a sequence of future observations and actions. For some test $t_h=(o_{h:h+W-1},a_{h:h+W-2})$ with length $W\in\BN^+$, we define the probability of test $t_h$ being successful conditioned on reachable history $\tau_{h-1}$ as $\BP(t_h|\tau_{h-1}):=\BP(o_{h:h+W-1}|\tau_{h-1};\doi(a_{h:h+W-2}))$, %\jnote{define the do stuff in terms of probabilities like Wen wrote in slack} 
i,e., the probability of observing $o_{h:h+W-1}$ by actively executing actions $a_{h:h+W-2}$ conditioned on history $\tau_{h-1}$.\footnote{The do operator means that $\BP(o_{h:h+W-1}|\tau_{h-1};\doi(a_{h:h+W-2}))= \prod_{t = h}^{h+W-1} \BP(o_t | \tau_{h-1}, o_{h:t-1}, a_{h:t-1})$. Here, we remark conditional probability of $o_{h:h+W-1}$ given $\tau_{h-1}$ is not specified not only by dynamics, but also by the policy. Given a policy $\pi$, conditional probability of $o_{h:h+W-1}$ given $\tau_{h-1}$ under a policy $\pi_t
$ is $\BP(o_{h:h+W-1}|\tau_{h-1};a_{h:h+W-2}\sim \pi) \propto \prod_{t= h}^{h+W-1} \BP( o_t | \tau_{t-1}) \pi_t(a_t | \tau_{t-2}, o_t)$. 
%%In contrast, note that in general $\BP(o_{h:h+W-1}|\tau_{h-1};a_{h:h+W-2}) \propto \prod_{t= h}^{h+W-1} \BP( o_t | \tau_{t-1}) \BP(a_t | \tau_{t-2}, o_t)$, i.e., it depends on an action policy. 
The $\doi( a_{h:h+W-2})$ operator can be understood as a policy that deterministically picks actions $a_t$ for $h\leq t\leq h+W-2$, i.e., $\pi_t(A_t = \cdot | \tau_{t-1}, o_t ) = \delta_{a_t}$. } % \wen{@ all: i added this new foonote for explaining the do operator, please check }\whz{looks good to me.}
When the history $\tau_{h-1}$ is unreachable, i.e., $\BP^{\pi}(\tau_{h-1})=0$ for all policy $\pi$, we define the conditional probability $\BP(t_h|\tau_{h-1})$ to be $0$. Now, we define the one-step system dynamics matrix $D_h$ whose rows  are indexed by tests and columns are indexed by histories, and the entry corresponding to the test-history pair $(t_{h+1}, \tau_h)$ is equal to $\BP(t_{h+1} | \tau_h)$ (see Fig~\ref{fig:core matrix} for an illustration).  Denote $\dpsh = \text{rank}(D_h)$. Then \emph{Linear PSRs} are defined to be systems with low-rank one-step system dynamic matrices:
\begin{definition}
\label{def:psr}
A partially observable system is called a Linear PSR with rank $\dps$ if $\max_h \text{rank}(D_h) = \dps$.
\end{definition}
% \jnote{Put the definition of PSR into the definition environment, and then the current equation (1) about $m^T P(u)$ can be a proposition that follows from $D_h$ being rank r.}\whz{\checkmark.}

For time step $h$, consider a set of tests $\UC_{h+1}\subset\cup_{C\in\mathbb{N}^{+}}\BO^C\times\A^{C-1}$. If the submatrix $\overline D_h$ (see the matrix inside the green box in Fig.~\ref{fig:core matrix}) of $D_h$ whose rows are indexed by the tests in $\UC_{h+1}$ and columns are indexed by all histories,  has rank equal to $\dpsh$, then we call such set $\UC_{h+1}$ as a \emph{core test set}.  The key property of such a core test set is that from linear algebras, for any row in $D_h$, we can express it as a linear combination of the rows of $\overline D_h$. This is formalized in the following lemma.

%\begin{definition}[Core test sets and linear PSRs]
%\label{def:psr}
%For any $h\in[H]$, a set $\UC_h\subset\cup_{C\in\mathbb{N}^{+}}\BO^C\times\A^{C-1}$ is called a core test set at $h$-th step if for any $W\in\mathbb{N}^{+}$, any possible future (i.e., test) $t_h=(o_{h:h+W-1},a_{h:h+W-2})\in\BO^W\times\A^{W-1}$ and any history $\tau_{h-1}$, there exists $m_{t_h,h}\in\mathbb{R}^{|\UC_h|}$ such that 
%\begin{align}
%\label{eq:psr def}
%\BP(t_h|\tau_{h-1})=\langle m_{t_h,h},[\BP(u|\tau_{h-1})]_{u\in\UC_h}\rangle.
%\end{align}
%\jnote{this equation can be a proposition that follows from the definition of a d dimensional linear psr.}
%The vector $[\BP(u|\tau_{h-1})]_{u\in\UC_h}$ is referred to as the predictive state at $h$-th step. 
%\end{definition}

\begin{lemma}[Core test sets in linear PSRs]
	\label{lem:core test}
	For any $h\in[H]$, a set $\UC_h\subset\cup_{C\in\mathbb{N}^{+}}\BO^C\times\A^{C-1}$ is a core test set at $h$-th step if and only if we have for any $W\in\mathbb{N}^{+}$, any possible future (i.e., test) $t_h=(o_{h:h+W-1},a_{h:h+W-2})\in\BO^W\times\A^{W-1}$ and any history $\tau_{h-1}$, there exists $m_{t_h,h}\in\mathbb{R}^{|\UC_h|}$ such that 
	\begin{align}
		\label{eq:psr def}
		\BP(t_h|\tau_{h-1})=\langle m_{t_h,h},[\BP(u|\tau_{h-1})]_{u\in\UC_h}\rangle.
	\end{align}
	The vector $[\BP(u|\tau_{h-1})]_{u\in\UC_h}$ is referred to as the predictive state at $h$-th step. 
\end{lemma}
%	\jnote{this equation can be a proposition that follows from the definition of a d dimensional linear psr.}\whz{\checkmark.}

Throughout the document, we use $q_{\tau_h}$ to denote $[\BP(u|\tau_{h})]_{u\in\UC_{h+1}}$ and $q_0$ to represent the initial predictive states $[\BP(u)]_{u\in\UC_1}$. In particular, we are interested in the set of all action sequences in $\UC_h$ and denote it by $\UAh$. A core test set with the smallest number of tests is called a minimum core test set, which we denote by $\Q_h$. Note that by the definition of the rank, we know that $|\Q_{h+1}|=\dpsh$. 
%and we define the minimum core test set size at $h$-th step as $\dpsh:=|\Q_h|$. 
To simplify writing, we further define $|\UC|:=\max_{h\in[H]}|\UC_h|, |\UA|:=\max_{h\in[H]}|\UAh|$. In this paper we assume a core test $\UC_h$ (we will see that this is a natural assumption for models such as POMDPs) is given while $\Q_h$ is unknown. This setting is standard in literature on PSRs \citep{boots2011closing}. 

\paragraph{Minimum core histories.} Similar to the minimum core test set, we can define minimum core history set as well. 
Consider the matrix $\BD_h$ in Figure~\ref{fig:core matrix}. Recall that the columns of $\BD_h$ are indexed by all possible $h$-length histories and each column is $q_{\tau_h}$. Since $\BD_h$ has rank $\dpsh$, there must exist $\dhi$ histories $\tau_h^{1},\cdots,\tau_{h}^{\dhi}$, such that any column in $\BD_h$ is a linear combination of the columns in $\BD_h$ that correspond to histories $\tau_h^{1},\cdots,\tau_{h}^{\dhi}$. In other words, 
%for any $h$-length history $\tau_h$
%Then from the definition of minimum core test set, we know the rank of $\BD_h$ is $\dhi$. This implies that there exists $\dhi$ histories $\tau_h^{1},\cdots,\tau_{h}^{\dhi}$ such that 
for any $h$-length history $\tau_h$, there exists a vector $v_{\tau_h}\in\R^{\dhi}$ which satisfies
\begin{align}
\label{eq:core history}
q_{\tau_h}=K_hv_{\tau_h},
\end{align} 
where $K_{h}\in\R^{|\UC_{h+1}|\times\dhi}$ is a full-rank matrix whose $i$-th column is $q_{\tau_h^i}$. We call $\{\tau_{h}^1,\cdots,\tau_{h}^{\dhi}\}$ as the minimum core histories at step $h$ and $K_h$ as the core matrix -- see Figure~\ref{fig:core matrix} for an illustration of $K_h$. Particularly, when $h=0$, we have $K_{0}=q_0$. Note \eqref{eq:core history} shows that all $h$-length histories can be captured by the core histories in the sense that the predicitive states given any history can be expressed as a linear combination of the predictive states corresponding to the minimum core histories. The minumum core histories and the core matrix may not be unique given the core test set. Here we particularly define $K_h$ to be the core matrix with smallest $\Vert K_h^{\dagger}\Vert_{1\mapsto1}$ to facilitate our subsequent analysis.

\begin{remark}
PSRs do not have latent states, unlike POMDPs. %In this case,
However, the minimum core histories can be viewed as the "\textit{states}" in PSRs because we can derive the system dynamics conditioned on any possible history from the dynamics conditioned on these core histories. As we will see, our algorithm does not require knowing the minimum core history set. The minimum core history set will only be used in the analysis. %not in our algorithm. 
\end{remark}

\paragraph{PSRs vs POMDPs.} PSRs have much stronger expressivity than POMDPs. All POMDPs can be expressed as PSRs with the minimum core test set size as most $|\SC|$ while PSRs are not necessarily compact POMDPs~\citep{littman2001predictive}. In Appendix~\ref{sec:rank} we construct a sequential decision making process where if we want to formulate it into a POMDP, the number of states we need will be exponentially larger than the core test set size in the PSR formulation. The key intuition behind the construction is simple: the non-negative rank of a matrix could be exponentially larger than its rank. In the literature \citep{singh2004predictive}, there are also some other concrete instances like probability clock which POMDPs cannot model with finite latent states while PSRs can model with finite rank. 

Here we illustrate how to characterize 1-step weakly-revealing POMDPs \citep{jin2020sample,liu2022partially} from the perspective of PSRs and defer other examples including $m$-step weakly-revealing POMDPs \citep{liu2022partially}, latent MDPs \citep{kwon2021rl}, $m$-step decodable POMDPs \citep{efroni2022provable} and low-rank POMDPs to Appendix~\ref{sec:example appendix}. Consider an episodic POMDP $(\SC,\BO,\A,\{\MT_h\}_{h=1}^H,\{\MO_h\}_{h=1}^H,\{r_h\}_{h=1}^H,H,\mu_1)$ where $\SC$ is the state space, $\BO$ is the observation space, $\A$ is the action space, $\MO_h$ is the emission matrix at $h$-th step, $r_h$ is the reward function at $h$-th step and $\mu_1$ is the initial state distribution. $\MT_h$ is the transition matrix at step $h$ where $(\MT_{h,a})_{s',s}=\BP_h(s'|s,a)$. Then the following lemma shows that any POMDP is a PSR and its minimum core test set size is no larger than $|\SC|$:
\begin{lemma}
\label{lem:pomdp dps}
All POMDPs satisfy the definition of PSRs (Definition~\ref{def:psr}). Further, we have $\dps\leq|\SC|$ for any POMDP. 
\end{lemma}
\begin{proof}
	\begin{figure}[t]
		\centering
		\includegraphics[width=\linewidth]{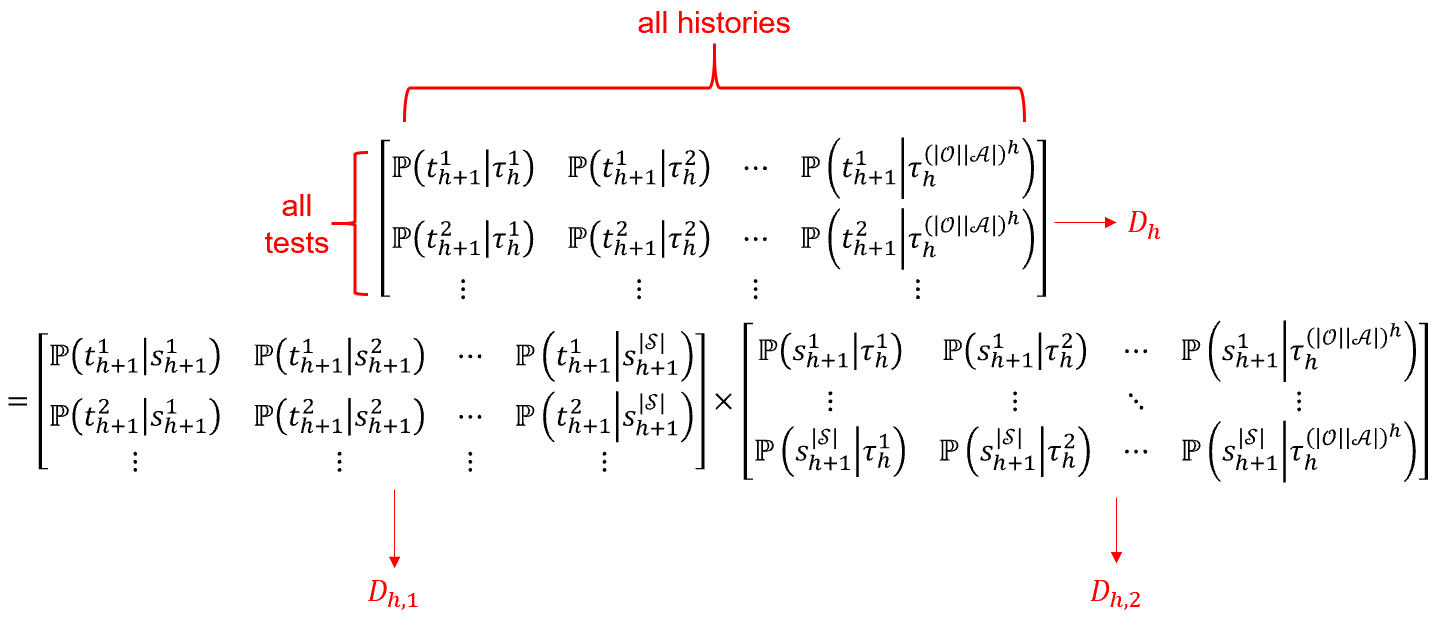}
		\caption{For any POMDP, the system dynamics matrix $D_h$ can always be factorized using the latent states. This factorization implies that the rank of $D_h$ is no larger than the number of latent states, which implies that POMDP is a linear PSR with rank at most equal to the number of latent states. Note that here $D_{h,1}$ and $D_{h,2}$ both contains non-negative entires. In contrast, from Figure~\ref{fig:core matrix}, the low-rank factorization of $D_h$ in PSR can have negative entries (i.e., $m$ and $v$ can have negative entries).}
		\label{fig:pomdp}
	\end{figure}
Consider the one-step system dynamics $D_h$ shown in Figure~\ref{fig:pomdp} whose rows are indexed by all possible future tests $t_{h+1}$ and columns are indexed by all histories $\tau_{h}$ at $h$-th step. Each entry of $D_h$ is the successful probability of the test, i.e., $(D_h)_{t_{h+1},\tau_{h}}=\BP(t_{h+1}|\tau_{h})$. Since we know $\BP(t_{h+1}|\tau_h)=\sum_{s_{h+1}\in\SC}\BP(t_{h+1}|s_{h+1})\BP(s_{h+1}|\tau_h)$ (where we also define $\BP(s_{h+1}|\tau_h)=0$ for unreachable $\tau_h$), we can decompose $D_h$ into the product of $D_{h,1}$ and $D_{h,2}$ as in Figure~\ref{fig:pomdp}, where $(D_{h,1})_{t_{h+1},s_{h+1}}=\BP(t_{h+1}|s_{h+1})$ and $(D_{h,2})_{s_{h+1},\tau_{h}}=\BP(s_{h+1}|\tau_{h})$. This implies that the rank of $D_h$ is not larger than $|\SC|$, which proves that it is a linear PSR with rank no larger than $|\SC|$. 
%\wen{I simplify the proof by ending it here -- basically this factorization implies PSR rank at most S -- please double check}
%\iffalse
%Therefore, there exist $|\SC|$ rows that can span the row space of $D_h$. Without loss of generality, suppose the tests of these $|\SC|$ rows are $t_{h+1}^1,\cdots,t_{h+1}^{|\SC|}$. Denote $q_{\tau_h}$ to be $[\BP(t_{h+1}^i|\tau_h)]_{i\in[|\SC|]}$. Then we have for any test $t_{h+1}$ and any history $\tau_{h}$, there exists a vector $m_{t_{h+1},h+1}$ such that
%\begin{align*}
%\BP(t_{h+1}|\tau_h)=m_{t_{h+1},h+1}^{\top}q_{\tau_h}.
%\end{align*}
%This shows that any POMDP is a PSR. Besides, since $\{t_{h+1}^i\}_{i\in[|\SC|]}$ is a core test set, we know $\dps\leq|\SC|$, which concludes our proof. \fi
\end{proof}

After showing any POMDP is a linear PSR with rank at most $|\SC|$, now we demonstrate that under what conditions we could find a core test set. We focus on 1-step weakly-revealing POMDPs \citep{jin2020sample,liu2022partially} here, i.e., the rank of $\MO_h$ is $|\SC|$ for all $h$, then we can show that $\BO$ is a core test set.
\begin{lemma}
\label{lem:under}
	When $\mathrm{rank}(\MO_h)=|\SC|$ for all $h\in[H]$, the POMDP is a PSR with the core test set $\UC_h=\BO$ for all $h\in[H]$. 
\end{lemma}
\begin{proof}
	Consider any $h\in[H]$, let $q_{\tau_{h-1}}=[\BP(o|\tau_{h-1})]_{o\in\BO}$. Then the belief state of the POMDP $\bs_{\tau_{h-1}}=[\BP(s_{h}|\tau_{h-1})]_{s_h\in\SC}$ can be expressed as:
	\begin{align*}
		\bs_{\tau_{h-1}}=\MO_{h}^{\dagger}q_{\tau_{h-1}}.
	\end{align*}
	Here, we use $\MO_{h}^{\dagger}\MO_{h}$ is an $|\SC| \times |\SC|$ identity matrix, which is verified by the assumption.
	Then for any test $t=(o_{h:h+W},a_{h:h+W-1})$, we know $\BP(t|\tau_{h-1})=m'_{t,h}\bs_{\tau_{h-1}}$ where
	\begin{align*}
		m'_{t,h}=\MO_{h+W}(o_{h+W}|\cdot)^{\top}\prod_{l=h}^{h+W-1}\MT_{l,a_l}\diag(\MO_l(o_l|\cdot)).
	\end{align*} 
	where $\MO_h(o|\cdot)\in\R^{|\SC|}$ is a vector whose $s$-th entry is $\MO_h(o|s)$ and $\mathbb{T}_{l,a_l}$ is a $|\mathcal{S}|\times |\mathcal{S}|$ matrix with entry $(\MT_{l,a_l})_{s',s}=\MT_l(s'|s,a_l)$. 
	
	Therefore we have $\BP(t|\tau_{h-1})=\langle m_{t,h},q_{\tau_{h-1}}\rangle$ where $m_{t,h}=(m'_{t,h}\MO_{h}^{\dagger})^{\top}$. Thus we have shown that the probability of any test $t$ is a linear combination of the probabilities of the tests $o\in \BO$ (the linear combination weights $m_{t,h}$ only depends on test but is independent of history).  
	This indicates that $\BO$ is a core test set for 1-step weakly-revealing POMDPs.
\end{proof}

In summary, we have shown that any POMDP is a PSR with $\dps \leq |\mathcal{S}|$, and if $\mathbb{O}$ is full rank (i.e., undercomplete POMDP), then $\mathcal{O}$ is a core test set. 

\subsubsection{Key Properties of PSRs}

\paragraph{Forward dynamics of predictive states.} Predictive states can evolve just like the beliefs in POMDPs. For any $o\in\BO,a\in\A,h\in[H]$, let $M_{o,a,h}\in\R^{|\UC_{h+1}|\times|\UC_h|}$ denote the matrix whose rows are $m_{(o,a,u),h}^{\top}$ (defined in Lemma~\ref{lem:core test}) for $u\in\UC_{h+1}$ (note that $o,a,u$ can be understood as a test that starts with $o,a$, followed by $u$). Then the following lemma charaterizes the forward update rule of the predictive states, whose proof is deferred to Appendix~\ref{proof:lem forward}:
\begin{lemma}[Forward dynamics of predictive state]
\label{lem:forward}
For any test $t$, $o\in\BO,a\in\A$ and reachable trajectory $(\tau_{h-1},o_h)$, we have
\begin{align*}
\BP(t|\tau_{h-1},o_h=o,a_h=a)=m_{(o,a,t),h}^{\top}q_{\tau_{h-1}}/m_{o,h}^{\top}q_{\tau_{h-1}}.
\end{align*}
In particular, we can express the forward update rule of the predictive states when $\tau_h=(\tau_{h-1},o,a)$ is reachable as follows:
\begin{align}
\label{eq:forward}
q_{\tau_{h-1},o,a}=M_{o,a,h}q_{\tau_{h-1}}/m_{o,h}^{\top}q_{\tau_{h-1}},
\end{align}
where we use $q_{\tau_{h-1},o,a}$ to denote $q_{\tau_h}$ with $\tau_h=(\tau_{h-1},o,a)$.
\end{lemma} The above forward predictive state update rule is analogous to the Bayesian filter in POMDPs. 
Lemma~\ref{lem:forward} implies that the predictive states $q_{\tau_h}$ with reachable history $\tau_h$ can be calculated from $M_{o_h,a_h,h}, m_{o_h,h}, q_0$. In particular, the probability of an arbitrary trajectory can be expressed as the product of the above parameters, as shown in the following lemma:
\begin{lemma}
\label{lem:product}
	For any trajectory $\tau_H$ and policy $\pi$, we have
	\begin{align}
	\label{eq:product 1}
		\BP^{\pi}(\tau_H)=m_{o_H,H}^{\top}\cdot\prod_{h=1}^{H-1}M_{o_h,a_h,h}\cdot q_0\cdot\pi(\tau_H),
	\end{align}
where $\pi(\tau_H):=\prod_{h=1}^H\pi(a_h|\tau_{h-1},o_h)$ is the probability of the actions chosen in the trajectory. More generally, for any $h\in[H]$ and trajectory $\tau_h$, we have
\begin{align}
\label{eq:product 2}
[ \BP(u|\tau_h)\BP^{\pi}(\tau_h)]_{u\in\UC_{h+1}}=b_{\tau_h}\times\pi(\tau_h),
\end{align}
where $b_{\tau_h}:=\{\prod_{l=1}^{h}M_{o_l,a_l,l}\}q_0$.
\end{lemma}
The proof is deferred to Appendix~\ref{proof:lem product}. Lemma~\ref{lem:product} shows that the parameters $$\{M_{o,a,h},m_{o,H},q_0\}_{o\in\BO,a\in\A,h\in[H-1]}$$ are sufficient to characterize a PSR. Here we call $M_{o,a,h}$ the predictive operator matrix.
%%%Recall that in POMDPs, we have similar decomposition of the trajectory probability with observable operator model \citep{jin2020sample,liu2022partially}. Here matrix $M_{o,a,h}$ can viewed as a generalization of observable operators in PSRs and we call $M_{o,a,h}$ the predictive operator matrix. 
Recall that in POMDPs, the same decomposition holds since we can represent $\{M_{o,a,h},m_{o,H},q_0\}$ using $\{\mathbb{T}_h,\mathbb{O}_h\}$ as we see in the proof of Lemma~\ref{lem:under}. However, as emphasized in \citet{singh2004predictive}, the main reason PSRs are more expressive is that $\{M_{o,a,h},m_{o,H},q_0\}$ are \emph{not} constrained to be non-negative. 
%\masa{i rephrase cause observable operators often refers to paramerziations PRS} 

%However, notice that \eqref{eq:forward} does not hold when $m_{o_h,h}^{\top}q_{\tau_{h-1}}=0$, suggesting that we might not be able to recover $q_{\tau_h}$ with unreachable history $\tau_h$. To deal with this issue, we denote the set of \textit{reachable predictive states} by $\bq_{\tau_h}$, which defines the conditional probability of unreachable histories to be $0$:
%\begin{align}
%	\label{eq:infer}
%	\bq_{\tau_h}=
%	\begin{cases}
%		0 & \text{ if $\BP^{\pi}(\tau_h)=0$ for any policy $\pi$,}\\
%		q_{\tau_h} & \text{ else.}
%	\end{cases}
%\end{align}
%It can be observed that $\bq_{\tau_h}$ satisfies the following forward dynamics:
%\begin{align}
%\label{eq:forward infer}
%\bq_{\tau_{h-1},o,a}=
%\begin{cases}
%M_{o,a,h}\bq_{\tau_{h-1}}/m_{o,h}^{\top}\bq_{\tau_{h-1}} & \text{ if $m_{o,h}^{\top}\bq_{\tau_{h-1}}>0$,}\\
%0 & \text{ else.}
%\end{cases}
%\end{align}

\subsubsection{Learning Objective} 

We sum up the notations in PSRs in Table~\ref{tab:notation}. In this paper, we study online learning in PSRs and want to find the optimal policy. Suppose the output policy is $\hpi$, then our goal is to find an $\eps$-optimal policy with polynomial number of samples such that:
\begin{align*}
V^*-V^{\hpi}\leq\eps,
\end{align*}
where $V^*:=V^{\pi^*}=\sup_{\pi}V^{\pi}$ and $\pi^*$ is the optimal policy.

\begin{table}[h!]
	\caption{Notations of PSRs. We also refer readers to Figure~\ref{fig:core matrix} for an illustration of the notations such as $\UC_h, \BD_h, D_h$, and $K_h$.}
	\begin{center}
		\label{tab:notation}
		\begin{tabular}{|c|c|}
			\hline
			Notation & Definition\\
			\hline
			$V^{\pi}$ & $\E_{\pi}\Big[\sum_{h=1}^Hr_h\Big]$\\
            \hline
            $\BP(t_h|\tau_{h-1})$ & $\BP(o_{h:h+W-1}|\tau_{h-1};\doi(a_{h:h+W-2}))$\\
            \hline
            $\UC_h$ & core test set at step $h$\\
            \hline
            $\UAh$ & the set of all action sequences in $\UC_h$\\
            \hline
            $q_{\tau_h}$ & predictive states $[\BP(u|\tau_{h})]_{u\in\UC_{h+1}}$\\
            \hline
            $m_{t_h,h}$ & $\BP(t_h|\tau_{h-1})=\langle m_{t_h,h},q_{\tau_{h-1}}\rangle$\\
            \hline
            $\dpsh$ & minimum core test set size at step $h$\\
            \hline
            $b_{\tau_h}$ & unnormalized predictive state $\{\prod_{l=1}^{h}M_{o_l,a_l,l}\}q_0$\\
            \hline
            $D_{h}$ & system dynamics \\
            \hline
            $\BD_h$ & predictive state dynamics whose columns are $q_{\tau_h}$\\
            \hline
            $K_h$ & core matrix at step $h$\\
            \hline 
            $M_{o,a,h}$ & predictive operator matrix\\
            \hline
		\end{tabular}
	\end{center}
\end{table}

\subsection{Function Approximation}
%To deal with the potentially large or even infinite observation and action space, we consider learning with function approximation in this paper. We assume a function class $\FC=\{\FC_h\}_{h=0}^{H-1}$ to approximate the parameters $\{M_{o,a,h},q_0\}_{o\in\BO,a\in\A,h\in[H-1]}$ in the PSR, where $\FC_h\subseteq\{f_h:\BO\times\A\to\R^{|\UC_{h+1}|\times|\UC_h|}\}$ for $h\in[H-1]$ and $\FC_0\subseteq\R^{|\UC_1|}$, so that we can obtain a model of system dynamics $\BP_f^{\pi}(\cdot)$ via \eqref{eq:product 1}. For any function $f\in\FC$, we use $M_{o,a,h;f}:=f_h(o,a)$ for $h\in[H-1]$ and $q_{0;f}:=f_0$ to denote the corresponding PSR parameters, $K_{h;f}$ to denote the core matrix induced by $M_{o,a,h;f}, q_{0;f}$, and $V_f^{\pi}$ to denote the value function specified by $f$. We also use $f^*$ to represent the true parameters for consistencty. 

To deal with the potentially large observation and action space, we consider learning with function approximation in this paper. We assume a function class $\FC$ to approximate the true model and let $\BP_{f}^{\pi}(\tau_H)$ denote the probability of any trajectory $\tau_H$ under any policy $\pi$ and model $f$. Here we assume that the models in $\FC$ are all valid PSRs with core test set $\{\UC_h\}_{h\in[H]}$, which implies that for each $f\in\FC$, we can calculate its corresponding predictive operator matrices, initial predictive states and core matrices, denoted by $M_{o,a,h;f},q_{0;f},K_{h;f}$ respectively. We define $V_f^{\pi}$ to be the value of policy $\pi$ under model $f$. We also use $f^*$ to represent the true model for consistency. 

Generally, we put models on $\{M_{o,a,h;f},q_{0:f}\}$ since this is the most natural parametrization in PSRs. When we have more prior knowledge about models like models are POMDPs, we can also put models on $\{\mathbb{T}_h,\mathbb{O}_h,\mu_1\}$.

To measure the size of $\FC$, we use $|\FC|$ to denote its cardinality when $\FC$ is finite. For infinite function classes, we introduce the $\eps$-bracket number to measure its size, which is defined as follows: 

\begin{definition}[$\eps$-bracket and $\eps$-bracket number]
A size-$N$ $\eps$-bracket is a bracket $\{g^i_1,g^i_2\}_{i=1}^N$ where $g^i_{1} (g^i_2)$ is a function mapping any policy $\pi$ and trajectory $\tau$ to $\R$ such that for all $i\in [N]$, $\Vert g^i_{1}(\pi,\cdot)-g^i_{2}(\pi,\cdot)\Vert_1\leq\eps$ for any policy $\pi$, and for any $f\in \FC$, there must exist an $i\in [N]$ such that $g^i_1(\pi,\tau_H) \leq \BP^{\pi}_{f}(\tau_H) \leq g^i_2(\pi,\tau_H)$ for all $\tau_H, \pi$. The $\eps$-bracket number of $\FC$, denoted by $\N_{\FC}(\eps)$, is the minimum size of such an $\eps$-bracket.
\end{definition}
Although $\BP^{\pi}_{f}$ is an $(|\BO||\A|)^H$-dimensional vector, its log $\eps$-bracket number will not scale exponentially with $H$ because $\BP^{\pi}_{f}$ is Lipschitz continuous with respect to $\{M_{o,a,h;f},q_{0;f}\}$, whose dimension only scales polynomially with $H$. In Appendix~\ref{sec:bracket} we show that the bracket number of $\FC$ can be upper bounded by the covering number of $\{M_{o,a,h;f},q_{0;f}\}$ in linear PSRs, and we provide exact upper bounds for tabular PSRs and various POMDPs.

\begin{remark}
The reason we need $\eps$-bracket number instead of $\eps$-covering number is that the confidence set of maximum log-likelihood estimators (MLE) for infinite function classes, which we will utilize in our algorithm, is characterized by the bracket number \citep{geer2000empirical}. That said, the log bracket number in our setting often only has an additional $H$ factor compared to log covering number, as shown in Appendix~\ref{sec:bracket}.
\end{remark}

\section{Algorithm: \mainalg}
\label{sec:alg}
The statistical hardness of learning POMDPs due to the partial observability is well-known in the literature \citep{krishnamurthy2016pac}, which also exists in PSR learning since PSRs are a more general model. In addition, existing algorithms \citep{jin2020sample,efroni2022provable,liu2022partially} for learning sub-classes of POMDPs require the existence of latent states since they directly put models on $\mathbb{T}$ and $\mathbb{O}$. Thus, their methods are not applicable to PSRs. That said, the existence of predictive states and their evolving dynamics \eqref{eq:forward} indeed imply the low-rank linear structure of PSRs. The trajectory probability decomposition \eqref{eq:product 1} further suggests that we are able to capture a PSR completely as long as we can learn the predictive operator matrices $\{M_{o,a,h}\}_{o\in\BO,a\in\A,h\in[H-1]}$ and the initial predictive state $q_0$ efficiently. Therefore, inspired form the success of maximum log-likelihood estimation (MLE) in learning weakly-revealing POMDPs \citep{liu2022partially}, we propose a new MLE-based PSR learning algorithm to learn these parameters as follows.

\paragraph{\mainalg.} Intuitively, our algorithm is an iterative MLE algorithm with optimism, where in each iteration we use MLE to estimate the model parameters based on the previously collected trajectories and choose an optimistic policy to execute. We call it Optimisti\underline{C} PS\underline{R} le\underline{A}rni\underline{N}g with ML\underline{E} (\mainalg). \mainalg mainly consists of three steps, whose details are shown in Algorithm~\ref{alg:main}:
\begin{itemize}
    \item \textbf{Optimism:} Since we consider the online learning problem, the unknown model dynamics force us to deal with the exploration-exploitation tradeoff. Here we utilize the \textit{Optimism in the Face of Uncertainty} principle %\citep{azar2017minimax,jin2020provably,jin2021bellman,huang2021towards} 
    and choose an optimistic estimation $f^k$ of the model parameters form the constructed confidence set $\B^k$. Our policy $\pi^k$ is the optimal policy under $f^k$, ensuring that $V^{\pi^k}_{f^k}\geq V^*$ with high probability. In this way, \mainalg will be able to trade off between exploration and exploitation. %prefer policies with more uncertainty and thus encourage exploration in the PSR. 
    
    \item \textbf{Trajectory collection:} For each step $h\in[H-1]^+$ and each action sequence $u_{a,h+1}$ in $\UAhp$, we collect a trajectory $\tau_{H}^{k,u_{a,h+1},h}$ by executing the policy $\pi^{k,u_{a,h+1},h}=\pi^k_{1:h-1}\circ\unif(\A)\circ u_{a,h+1}$ (and uniform policy afterwards if the episode has not ended). This helps us obtain the required information for estimating each predictive operator matrix $M_{o,a,h}$ and initial predictive state $q_0$. 
    		
	\item \textbf{Parameter estimation with MLE:} Finally we need to update the confidence set with newly collected trajectories. We achieve this by implementing MLE on all the collected trajectories with slackness $\beta$, as shown below:
	\begin{align}
	\label{eq:update}
	\B^{k+1}\gets\bigg\{f\in\FC:\sum_{(\pi,\tau_H)\in\D}\log\BP^{\pi}_{f}(\tau_H)\geq\max_{f'\in\FC}\sum_{(\pi,\tau_H)\in\D}\log\BP^{\pi}_{f'}(\tau_H)-\beta\bigg\}.
	\end{align}
	For example, the likelihood $\BP^{\pi}_{f}(\tau_H)$ is specified by \eqref{eq:product 1} if we model $\{M_{o,a,h},q_{0}\}$ as $f$. In POMDPs, if we model $\{\mathbb{T}_h,\mathbb{O}_h\}$ as $f$, the likelihood is specified by marginalizing over latent states. The slackness $\beta$ is chosen appropriately such that the true parameters $f^*\in\B^{k+1}$ with high probability, which in turn guarantees optimism in the first step.  
\end{itemize}

\begin{algorithm}[ht]
	\caption{\textbf{\mainalg}}
	\label{alg:main}
	\begin{algorithmic}
	\State \textbf{Input}: confidence parameter $\beta$.
	\State Initialize $\B^1\gets\FC$, $\D=\emptyset$.
	\For{$k=1,\cdots,K$}
	\State \textbf{Optimistic Planning: } $(f^k,\pi^k)\gets\arg\max_{f\in\B^k,\pi}V^{\pi}_{f}$.
	\State \textbf{Collect samples: }
	\For{$h\in[H-1]^+,u_{a,h+1}\in\UAhp$}
	\State Exectute $\pi^{k,u_{a,h+1},h}=\pi^k_{1:h-1}\circ\unif(\A)\circ u_{a,h+1}$ and collect a trajecotry $\tau_{H}^{k,u_{a,h+1},h}$.
	\State Update the dataset $\D\gets\D\cup(\pi^{k,u_{a,h+1},h},\tau_{H}^{k,u_{a,h+1},h})$.
	\EndFor
	\State \textbf{Update confidence set: }Compute $\B^{k+1}$ via \eqref{eq:update}. 
	\EndFor
	\end{algorithmic}
\end{algorithm}

\paragraph{Comparision with \citet{liu2022partially}.}
	The main difference is that our algorithm can allow more general models. For example, in PSRs, we can generally take $\{M_{o,a,h},q_0 \}$ that depends on only observable quantities as a model $f$. On the other hand, \citet{liu2022partially} attempts to put models on $\{\mathbb{T}_h,\mathbb{O}_h\}$ that involve latent states. The practical benefit of modeling $\{M_{o,a,h},q_0 \}$ is we don't need to specify the latent space while \citet{liu2022partially} needs to do. Since we often do not have good prior knowledge about latent states, our algorithm would be more practical in this scenario. %When we have prior knowledge about latent states, we can also start from $\{\mathbb{T}_h,\mathbb{O}_h\}$. 
	Due to the generality of our algorithm, we can capture more models such as $m$-step decodable POMDPs and low-rank POMDPs as we will see in the following sections. Another difference is we only execute the action sequences in the core test set, i.e., $\UAh$, actively. On the other hand, \citet{liu2022partially} executes the whole $m$-step futures.

\section{Main result}\label{sec:main}
\label{sec:main result}
Next, we present the regret analysis for \mainalg. We will utilize the fact that the core matrix $K_h$ is full-rank. However, matrix rank is vulnerable to estimation errors since small perturbations might change the rank drastically. Here we assume the $\ell_1$ norm of $K_h^{\dagger}$ is upper bounded, which is a more robust assumption than $K_h$ being full rank. Note a similar assumption is often imposed in the PSR literature \citep[Appendix B.4]{jiang2017contextual}. 

%the minimum non-zero singular value of $K_h$ is lower bounded:
\begin{assumption}[$\alpha$-regularity of PSRs.]
	\label{ass:regular}
	Assume that there exists $\alpha>0$ such that for any $h\in[H-1]^+$, we have $\Vert K_h^{\dagger}\Vert_{1\mapsto1}\leq1/\alpha$.
\end{assumption}

\begin{remark}
Notice that $\Vert K_h^{\dagger}\Vert_{1\mapsto1}$ can be upper bounded by $\sqrt{\dhi}/\sigma_{\min}(K_h)$. In the literature of POMDPs, many works \citep{jin2020sample,efroni2022provable,liu2022partially} assume a similar condition called $\alpha$-weakly revealing condition. That is, the minimal singular value of the observation matrix or the multi-step observation matrix is lower bounded by $\alpha$. Assumption~\ref{ass:regular} can be regarded as a generalization of such weakly revealing condition in PSRs by viewing core histories as the "\textit{states}" and tests $u\in\UC_h$ as the "\textit{observations}" in PSRs.
\end{remark}

In addition, to simplify analysis, we assume all one-step observations $o\in\BO$ belongs to $\UC_H$. This does not harm the generality of our model since augmenting the core test set is always feasible and adding all one-step observations will at most increase $|\UA|$ by one (i.e., adding a null action sequence).
\begin{assumption}
\label{ass:observe}
For all $o\in\BO$, we assume that $o\in\UC_H$. 
\end{assumption}
This assumption immediately implies that $m_{o,H} = e_{o,H}$, i.e., it is a one-hot vector which indexes the observation $o$ in $\UC_H$. To see that, note that $o\in \UC_H$ implies that the predictive state $q_{\tau_{H-1}}$ contains the probability $\BP(o | \tau_{H-1})$. Thus, when $m_{o,H} = e_{o,H}$, we have $m_{o,H}^\top q_{\tau_{H-1}} = \BP(o | \tau_{H-1})$. Therefore when Assumption~\ref{ass:observe} holds, we can assume that for all models induced by $\FC$, we have $m_{o,H;f} = e_{o,H}$ for all $o$ without loss of generality. %\whz{We only need $o\in\UC_H$?}

Furthermore, we also need to impose constraints on the function class $\FC$ so that it is well-specified for PSR learning:
\begin{assumption}
\label{ass:function}
Assume the function class $\FC$ satisfies the following conditions:
\begin{itemize}
	\item \textbf{Realizability:} We have $f^*\in\FC$.
	\item \textbf{Regularity:} For all $f\in\FC$ and $h\in[H-1]^+$, we have $\Vert K_{h;f}^{\dagger}\Vert_{1\mapsto1}\leq1/\alpha$.
%	\item \textbf{Non-redundancy:} For all $f\in\FC, o\in\BO, a\in\A, h\in[H-1], u\in\UC_{h+1}$, we have $m_{(o,a,u),h;f}$ belongs to the column space of $K_{h-1;f}$.
	\item \textbf{Validity:} For all $f\in\FC$, the model dynamics induced by $f$ is a valid PSR with core test set $\{\UC_h\}_{h\in[H]}$, i.e., the trajectory probability $\BP^{\pi}_f$ should be a valid distribution for any policy $\pi$ and satsifies the definition of PSRs.
\end{itemize}
\end{assumption} 
 The last two constraints in Assumption~\ref{ass:function} can be easily satisfied by eliminating those functions which do not satisfy the regularity or validity.

Notice that the forward dynamics in \eqref{eq:forward} only utilizes the inner product of $m_{(o,a,u),h;f}$ and $q_{\tau_{h-1};f}$, and $q_{\tau_{h-1};f}$ lives in the column space of $K_{h-1;f}$ (i.e., \eqref{eq:core history}), which implies there is redundancy in the choice of $m_{(o,a,u),h;f}$ given the model $\BP^{\pi}_f$. Next we show that among these possible $m_{(o,a,u),h;f}$, we can always find one that lies in the column space of $K_{h-1;f}$. More specifically, if we replace any $m_{(o,a,u),h;f}$ with its projection on the space spanned by $\{q_{\tau_{h-1};f}\}_{\tau_{h-1}}$ (which is exactly the column space of $K_{h-1;f}$), the resulting model dynamics will remain the same. The following lemma gives a more formal statement about this fact, whose proof is deferred to Appendix~\ref{proof:lem redundancy}:
\begin{lemma}
\label{lem:redundancy}
Suppose Assumption~\ref{ass:observe} holds. Given any parameter $\{M_{o,a,h;f},q_{0;f}\}$, suppose for another set of parameters $\{M_{o,a,h;f'},q_{0;f}\}$ we have for all $o\in\BO,a\in\A,h\in[H-1],u\in\UC_{h+1}$
\begin{align*} 
m_{(o,a,u),h;f'}=\proj_{\col(K_{h-1;f})}(m_{(o,a,u),h;f}). 
\end{align*}
Then for any trajectory $\tau_H$ and policy $\pi$, we have
\begin{align*}
\BP^{\pi}_{f}(\tau_H)=\BP^{\pi}_{f'}(\tau_H).
\end{align*} 
This means that $\{m_{(o,a,u),h;f'}\}$ is also a valid set of predictive parameters for the model $\BP_f$. 
\end{lemma} 
Lemma~\ref{lem:redundancy} indicates that given any model $\BP^{\pi}_f$ where $f\in\FC$, there always exsits a set of $\{m_{(o,a,u),h;f}\}_{o\in\BO,a\in\A,u\in\UC_{h+1},h\in[H-1]}$ such that $m_{(o,a,u),h;f}$ belongs to the colunm space of $K_{h-1;f}$. Therefore in the following discussion, we let $M_{o,a,h;f}$ consist of such $m_{(o,a,u),h;f}$ without loss of generality.

 With the above assumptions, we have Theorem~\ref{thm:main} to show that \mainalg can achieve sublinear total suboptimality, whose proof is deferred to Section~\ref{sec:sketch}.
\begin{theorem}
\label{thm:main}
Under Assumption~\ref{ass:regular},\ref{ass:observe},\ref{ass:function}, there exists an absolute constant $c$ such that for any $\delta\in(0,1]$, $K\in\mathbb{N}$, if we choose $\beta=c\log(\ncov KH|\UA|/\delta)$ in \mainalg where $\eb=1/(KH|\UA|)$, then with probability at least $1-\delta$, we have:
\begin{align*}
%\label{eq:thm1}
\sum_{k=1}^K(V^*-V^{\pi^k})\leq\BO(\dps^2H^{\frac{7}{2}}|\UA|^4|\A|^2K^{\frac{1}{2}}\alpha^{-3}\cdot\log(KH\ncov|\BO||\A|/\delta)).
\end{align*}
\end{theorem}

The $\sqrt{K}$ bound on the regret in Theorem~\ref{thm:main} suggests that the uniform mixture of the output policies $\hpi=\unif(\{\pi^k\}_{k=1}^K)$ is an $\eps$-optimal policy when $K=\TO(1/\eps^2)$, leading to the following corollary which characterizes the sample complexity of \mainalg to learn a near-optimal policy:
\begin{corollary}
\label{cor:main}
Under Assumption~\ref{ass:regular},\ref{ass:observe},\ref{ass:function}, for any $\delta\in(0,1],\eps>0$, if we choose
\begin{align*}
%	\label{eq:cor1} 
%	K=\TO(\dps^4H^{7}|\UA|^8|\A|^4\alpha^{-6}\eps^{-2}),
K=1/\epsilon^2\times\poly(\dps,|\UA|,1/\alpha,\log\ncov,H,|\A|,\log|\BO|,\log(1/\delta)),
\end{align*} 
and set $\beta=c\log(\ncov KH|\UA|/\delta)$, then with probability at least $1-\delta$ we have
\begin{align*}
V^{\hpi}\geq V^*-\eps.
\end{align*}
\end{corollary}

 Theorem~\ref{thm:main} and Corollary~\ref{cor:main} indicate that the complexity of \mainalg only depends polynomially on the PSR rank $\dps$, the size of $\UAh$, the $\ell_1$-norm of the pseudoinverse of the core matrix $\frac{1}{\alpha}$, the log bracket number of function classes $\log \ncov$, $H$ and $|\A|$.  \mainalg avoids direct dependency on poly($|\BO|$) and our sample complexity remains the same even if the observation parts in core test set $\mathcal{U}_h$ is large. Via the relationship between POMDPs and PSRs, Theorem~\ref{thm:main} can be applied to $m$-step weakly-revealing tabular POMDPs (including undercomplete POMDPs \citep{jin2020sample} and overcomplete POMDPs \citep{liu2022partially}), $m$-step wealy-revealing low-rank POMDPs \citep{wang2022embed}, $m$-step weakly-revealing linear POMDPs and $m$-step decodable POMDPs \citep{efroni2022provable}, which we will elaborate on in Section~\ref{sec:example}.. 
 % our algorithm is amesnable to setting when $|\BO$ is extremely large.  }

Our sample complexity in Corollary~\ref{cor:main} depends on the upper bound of $\Vert K_h^{\dagger}\Vert_{1\mapsto1}$, i.e., $1/\alpha$. In the following theorem we show that such dependence is inevitable in general:
\begin{theorem}
\label{thm:lower bound}
For any $0<\alpha<\frac{1}{2\sqrt{2}}$ and $H,|\A|\in\mathbb{N}^{+}$, there exists a PSR with core test set $\UC_h=\BO$ for $h\in[H]$ and $|\SC|=|\BO|=\BO(1)$ which satisfies Assumption~\ref{ass:regular} so that any algorithm requires at least $\Omega(\min\{\frac{1}{\alpha H},|\A|^{H-1}\})$ samples to learn a $(1/2)$-optimal policy with probability $1/6$ or higher.
\end{theorem}
Theorem~\ref{thm:lower bound} indicates that scaling with $1/\alpha$ is unavoidable or else the algorithm will require an exponential number of samples to learn a near optimal policy. The proof is deferred to Appendix~\ref{proof:thm lower bound}.

\paragraph{Proof techniques of Theorem~\ref{thm:main}.} The existing analysis for POMDPs is not applicable to PSRs since we do not assume latent states in PSRs, let alone the emission matrix and transition matrix. In our proof, we utilize the linear nature of PSRs and leverage the core matrix $K_h$ to bound the error propagation induced by the product of predictive operator matrices, i.e., $\prod_{h=1}^{H-1}M_{o_h,a_h,h}$ in \eqref{eq:product 1}. This key step enables us to bound the difference of model dynamics $\BP_f^{\pi}$ and $\BP^{\pi}$ (i.e., $\BP^{\pi}_{f^{\star}}$ ) by the estimation error of $M_{o_h,a_h,h}$ and $q_0$, and thus obtain a polynomial bound on the total suboptimality. This also implies the core matrix $K_h$ that characterizes the relationships between the histories and the future may be a more intrinsic quantity than the emission matrix. As byproduct of our analysis, even in POMDPs, we can capture models with low-rank latent transition that \citet{liu2022partially} cannot capture. %\whz{\checkmark.}

\paragraph{Comparison with existing works on PSRs.} As far as we know, there are only two works that tackle provably efficient RL for PSRs. \citet{jiang2016contextual} shows a polynomial sample complexity result in reactive POMDPs where optimal value functions only depend on current observations. Later, \citet{uehara2022provably} shows a favorable sample complexity result without this assumption. However, their result is an agnostic-type result that depends on $(|\BO||\A|)^M$ when competing with $M$-memory policies. Thus, to compete with the globally optimal policy, their results do \emph{not} imply a polynomial sample complexity bound. %%%%In addition, their sample complexity bounds depend on $|\mathcal{U}_{O}|$ where $|\mathcal{U}_O|$ is the size of the observation part in core histories. Our result does \emph{not} depend on  $|\mathcal{U}_{O}|$. \whz{add directly? since in tabular PSRs the bracket number depends on $|\UC|$.} 

We note existing works \citep{boots2013hilbert,hefny2015supervised,boots2010hilbert,kulesza2014low} on the system identification perspective of PSRs explicitly assume the knowledge of short core histories that include minimum core histories. On the other hand, we do \emph{not} require this assumption since our algorithm is agnostic to minimum core histories. Minimum core histories are just used for analytical purposes.

\section{Proof Sketch of Theorem~\ref{thm:main}}
\label{sec:sketch}
In this section we present a proof sketch for Theorem~\ref{thm:main}. Note that we can decompose the regret into the following terms:
\begin{align}
	\reg(K)&=\sum_{k=1}^K\big(V^*-V^{\pi^k}\big)=\bigg(\underbrace{\sum_{t=1}^{K}\big(V^*-V_{f^k}^{\pi^k}\big)}_{\displaystyle(1)}\bigg) +\bigg(\underbrace{\sum_{t=1}^{K}\big(V_{f^k}^{\pi^k}-V^{\pi^k}\big)}_{\displaystyle(2)}\bigg).\label{eq:decompose}
\end{align}

Our proof bounds these two terms separately and mainly consists of four steps:
\begin{enumerate}
\item Prove $V^{\pi^k}_{f^k}$ is an optimistic estimation of $V^*$ for all $k\in[K]$, which implies that term $(1)\leq0$.
\item Decompose term (2) into the estimation error of the parameter $M_{o,a,h}$ via the forward dynamics \eqref{eq:forward}.
\item Bound the cumulative estimation error using the property of MLE. \item Bound term (2) by connecting the results in the second and third step.
\end{enumerate}

\subsection{Step 1: Prove Optimism}
First we can show that the constructed set $\B^k$ contains the true model parameter $f^*$ with high probability:
\begin{lemma}
\label{lem:optimism}
With probability at least $1-{\delta}/{2}$, we have for all $k\in[K]$, $f^*\in\B^k$.	
\end{lemma}
\begin{proof}
See Appendix~\ref{proof:optimism}.
\end{proof}
Then since $V_{f^k}^{\pi^k}=\max_{f\in\B^k,\pi}V_f^{\pi}$, we know for all $k\in[K]$,
\begin{align*}
V_{f^k}^{\pi^k}\geq V_{f^*}^{\pi^*}=V^*.
\end{align*}
Thus, Lemma~\ref{lem:optimism} implies that $V_{f^k}^{\pi^k}$ is an optimistic estimation of $V^*$ for all $k$, and therefore term (1) in \eqref{eq:decompose} is non-positive.

\subsection{Step 2: Decompose the Performance Difference} 
Next we aim to handle term (2) in  \eqref{eq:decompose} and show the estimation error $\sum_{t=1}^{K}\big(V_{f^k}^{\pi^k}-V^{\pi^k}\big)$ is small. First we need to decompose the performance difference $V_{f^k}^{\pi^k}-V^{\pi^k}$ into the estimation error of the parameters $M_{o,a,h}$ in order to apply the property of MLE later. Notice that we have,
\begin{align}
&V_{f^k}^{\pi^k}-V^{\pi^k}\leq H\sum_{\tau_H}|P^{\pi^k}_{f^k}(\tau_H)-P^{\pi^k}_{f^*}(\tau_H)|\notag\\
&\qquad=H\sum_{\tau_H}\bigg|(m^k_{o_H,H})^{\top}\cdot\prod_{h=1}^{H-1}M^k_{o_h,a_h,h}\cdot q^k_0-m_{o_H,H}^{\top}\cdot\prod_{h=1}^{H-1}M_{o_h,a_h,h}\cdot q_0\bigg|\times\pi^k(\tau_H)\notag\\
&\qquad= H\sum_{\tau_{H-1}}\bigg\Vert\prod_{h=1}^{H-1}M^k_{o_h,a_h,h}\cdot q^k_0-\prod_{h=1}^{H-1}M_{o_h,a_h,h}\cdot q_0\bigg\Vert_1\times\pi^k(\tau_{H-1}),\label{eq:difference 1}
\end{align}
where we use $m^k_{o_H,H},M^k_{o_h,a_h,h},q^k_0$ to denote $m_{o_H,H;f^k},M_{o_h,a_h,h;f^k},q_{0;f^k}$ and $m_{o_H,H},M_{o_h,a_h,h},q_0$ to denote $m_{o_H,H;f^*},M_{o_h,a_h,h;f^*},q_{0;f^*}$. The second step is due to Lemma~\ref{lem:product} and the last step is because based on Assumption~\ref{ass:observe} we have set $m^k_{o_H,H}=m_{o_H,H}=e_{o_H,H}$. 

The following lemma bridges the term in \eqref{eq:difference 1} and the estimation error of $M_{o,a,h}$, whose proof is deferred to Appendix~\ref{proof:performance}:
\begin{lemma}
\label{lem:performance}
For any $k\in[K]$, $h\in[H]$ and policy $\pi$, we have
\begin{align}
&\sum_{\tau_h}\bigg\Vert\prod_{l=1}^{h}M^k_{o_l,a_l,l}\cdot q^k_0-\prod_{l=1}^{h}M_{o_l,a_l,l}\cdot q_0\bigg\Vert_1\times\pi(\tau_h)\notag\\
&\qquad\leq\frac{|\UA|}{\alpha}\bigg(\sum_{l=1}^{h}\sum_{\tau_l}\Vert[M^k_{o_l,a_l,l}-M_{o_l,a_l,l}]b_{\tau_{l-1}}\Vert_1\times\pi(\tau_l)+\Vert q^k_0-q_0\Vert_1\bigg),\label{eq:performance 1}
\end{align}
where $b_{\tau_{l}}=\prod_{j=1}^lM_{o_j,a_j,j}q_0$.
\end{lemma}
\begin{remark}
From the proof of Lemma~\ref{lem:performance}, we can see that Lemma~\ref{lem:performance} only utilizes the properties of $\{M^k_{o_l,a_l,l},q^k_0\}_{l\in[h]}$. Therefore Lemma~\ref{lem:performance} still holds even if the system dynamics induced by $\{M_{o_l,a_l,l},q_0\}_{l\in[h]}$ is invalid. We will use this fact in the analysis about the $\eps$-bracket number.
\end{remark}
Therefore, substituting Lemma~\ref{lem:performance} into \eqref{eq:difference 1}, we can bound the performance difference by the cumulative estimation error:
\begin{align}
\label{eq:difference 2}
&\sum_{k=1}^{K}\big(V_{f^k}^{\pi^k}-V^{\pi^k}\big)\notag\\
&\qquad\leq\frac{|\UA|H}{\alpha}\sum_{k=1}^K\bigg(\sum_{h=1}^{H-1}\sum_{\tau_h}\Vert[M^k_{o_h,a_h,h}-M_{o_h,a_h,h}]b_{\tau_{h-1}}\Vert_1\times\pi^k(\tau_h)+\Vert q^k_0-q_0\Vert_1\bigg).
\end{align}

\subsection{Step 3: Bound the Estimation Error}
Now we need to bound the estimation error in \eqref{eq:difference 2}. First we introduce the following guarantee of MLE from the literature, which connects the log-likelihood ratio $\log(P^{\pi}_{f^*}(\tau_H)/P^{\pi}_{f}(\tau_H))$ and the total variation $\sum_{\tau_H}|P^{\pi}_{f}(\tau_H)-P^{\pi}_{f^*}(\tau_H)|$:
\begin{lemma}[{\cite[Proposition 14]{liu2022partially}}]
\label{lem:estimation}
There exists a universal constant $c_1$ such that for any $\delta\in(0,1]$, for all $k\in[K]$ and $f\in\FC$, we have with probability at least $1-\delta/2$ that
\begin{align*}
&\sum_{i=1}^k\sum_{h\in[H-1]^+,u_{a,h+1}\in\UAhp}\bigg(\sum_{\tau_H}|\BP^{\pi^{i,u_{a,h+1},h}}_f(\tau_H)-\BP^{\pi^{i,u_{a,h+1},h}}_{f^*}(\tau_H)|\bigg)^2\\
&\qquad\leq c_0\bigg(\sum_{i=1}^k\sum_{h\in[H-1]^+,u_{a,h+1}\in\UAhp}\log\bigg(\frac{\BP^{\pi^{i,u_{a,h+1},h}}_{f^*}(\tau^{i,u_{a,h+1},h}_H)}{\BP^{\pi^{i,u_{a,h+1},h}}_{f}(\tau^{i,u_{a,h+1},h}_H)}\bigg)+\log(\ncov KH|\UA|/\delta)\bigg).
\end{align*} 
\end{lemma}

Combining Lemma~\ref{lem:estimation} and the fact that both $f^*$ and $f^k$ belongs to $\B^k$, 
we have with probability at least $1-\delta$ that for all $k\in[K]$,
\begin{align}
\label{eq:estimation 1}
\sum_{i=1}^{k-1}\sum_{h\in[H-1]^+,u_{a,h+1}\in\UAhp}\bigg(\sum_{\tau_H}|\BP^{\pi^{i,u_{a,h+1},h}}_{f^k}(\tau_H)-\BP^{\pi^{i,u_{a,h+1},h}}_{f^*}(\tau_H)|\bigg)^2\leq\BO(\beta).
\end{align}
The following discussion is conditioned on the event in \eqref{eq:estimation 1} being true. Then by Cauchy-Schwarz inequality we have for all $k\in[K]$,
\begin{align}
\label{eq:estimation 2}	
\sum_{i=1}^{k-1}\sum_{h\in[H-1]^+,u_{a,h+1}\in\UAhp}\sum_{\tau_H}|\BP^{\pi^{i,u_{a,h+1},h}}_{f^k}(\tau_H)-\BP^{\pi^{i,u_{a,h+1},h}}_{f^*}(\tau_H)|\leq\BO(\sqrt{kH|\UA|\beta}).
\end{align}

Suppose the length of the longest action sequence in $\UAh$ is $l_a$. Then since the environment will only generate dummy observations $\od$ after $a_H$, we have for any policy $\pi$ and $f\in\FC$,
\begin{align*}
&\sum_{\tau_{H+l_a+1}}|\BP^{\pi}_{f}(\tau_{H+l_a+1})-\BP^{\pi}_{f^*}(\tau_{H+l_a+1})|\\
&\qquad=\sum_{\tau_{H+l_a+1}}|\BP^{\pi}_{f}(\tau_{H})\mathbf{1}(o_{H+1:H+l_a+1}=\od)\pi(a_{H+1:H+l_a+1}|\tau_H)\\
&\qquad\quad-\BP^{\pi}_{f^*}(\tau_{H})\mathbf{1}(o_{H+1:H+l_a+1}=\od)\pi(a_{H+1:H+l_a+1}|\tau_H)|\\
&\qquad=\sum_{\tau_{H}}|\BP^{\pi}_{f}(\tau_{H})-\BP^{\pi}_{f^*}(\tau_{H})|.
\end{align*}

Therefore, we can marginalize the distribution $\BP^{\pi}_{f}(\tau_{H})$ and $\BP^{\pi}_{f^*}(\tau_{H})$ in \eqref{eq:estimation 2} and have for all $k\in[K],i\in[k-1],h\in[H-1]^+$, 
%\wen{add more step on how this marginalization is done?}
%\wen{below, it should be $\pi^{i, u_a, h}$? i.e., why $u_{a, h+1}$?}\jnote{agree, more steps/algebra. and say what we are marginalizing over}\whz{added multiple steps.}
\begin{align*}
&\sum_{u_{a,h+1}\in\UAhp}\sum_{\tau_H}|\BP^{\pi^{i,u_{a,h+1},h}}_{f^k}(\tau_H)-\BP^{\pi^{i,u_{a,h+1},h}}_{f^*}(\tau_H)|\\
&\quad\geq\sum_{u_{a,h+1}\in\UAhp}\sum_{\tau_h,\boo\in\BOO(u_{a,h+1})}|\BP^{\pi^{i,u_{a,h+1},h}}_{f^k}(\tau_h,u_{a,h+1},\boo)-\BP^{\pi^{i,u_{a,h+1},h}}_{f^*}(\tau_h,u_{a,h+1},\boo)|\\
&\quad=\sum_{u_{a,h+1}\in\UAhp}\sum_{\tau_h,\boo\in\BOO(u_{a,h+1})}|\BP^{\pi^{i,h}}_{f^k}(\tau_h)\BP_{f^k}(\boo|\tau_h;\doi(u_{a,h+1}))\\
&\quad\quad-\BP^{\pi^{i,h}}_{f^*}(\tau_h)\BP_{f^*}(\boo|\tau_h;\doi(u_{a,h+1}))|\times\pi^{i,u_{a,h+1},h}(u_{a,h+1}|\tau_h)\\
&\quad=\sum_{\tau_h}\sum_{u_{a,h+1}\in\UAhp,\boo\in\BOO(u_{a,h+1})}|\BP^{\pi^{i,h}}_{f^k}(\tau_h)\BP_{f^k}(\boo|\tau_h;\doi(u_{a,h+1}))-\BP^{\pi^{i,h}}_{f^*}(\tau_h)\BP_{f^*}(\boo|\tau_h;\doi(u_{a,h+1}))|\\
&\quad=\sum_{\tau_h}\sum_{u\in\UC_{h+1}}|\BP^{\pi^{i,h}}_{f^k}(\tau_{h})\BP_{f^k}(u|\tau_h)-\BP^{\pi^{i,h}}_{f^*}(\tau_{h})\BP_{f^*}(u|\tau_h)|.\\
\end{align*}
Here in the first step $\BOO(u_{a,h+1})$ denote the set of observation sequences that occur together with $u_{a,h+1}$ in $\UC_{h+1}$ and $\BP^{\pi}_f(\tau_h,u_{a,h+1},\boo)$ denotes the joint probability of observing the trajectory $(\tau_h,\boo,u_{a,h+1})$. In the third and fourth step we utilize the fact that $\pi^{k,u_{a,h+1},h}=\pi^k_{1:h-1}\circ\unif(\A)\circ u_{a,h+1}$ and we define $\pi^{i,h}:=\pi^i_{1:h-1}\circ\unif(\A)$. Then based on Eq.~\eqref{eq:product 2} and Eq.~\eqref{eq:estimation 2}, we have for all $k\in[K],h\in[H-1]^+$, 
\begin{align*}
\sum_{i=1}^{k-1}\sum_{\tau_h}\pi^{i,h}(\tau_h)\cdot\Vert b^k_{\tau_h}-b_{\tau_h}\Vert_1\leq\BO(\sqrt{kH|\UA|\beta}).
\end{align*}
Thus via importance weighting, we have for all $k\in[K],h\in[H-1]^+$,
\begin{equation}
\label{eq:estimation 3}
\sum_{i=1}^{k-1}\sum_{\tau_h}\pi^{i}(\tau_h)\cdot\Vert b^k_{\tau_h}-b_{\tau_h}\Vert_1\leq\BO(|\A|\sqrt{kH|\UA|\beta}),
\end{equation}
\begin{equation}
	\label{eq:estimation 4}
	\sum_{i=1}^{k-1}\sum_{\tau_h}\pi^{i}(\tau_{h-1})\cdot\Vert b^k_{\tau_h}-b_{\tau_h}\Vert_1\leq\BO(|\A|\sqrt{kH|\UA|\beta}).
\end{equation}

In particular, when $h=0$ we have 
\begin{align}
\label{eq:estimation 6}
\Vert q^k_0-q_0\Vert_1\leq\BO(\sqrt{H|\UA|\beta/k})
\end{align}

Now for all $k\in[K],h\in[H-1]$, we can bound the estimation error as follows:
\begin{align}
 &\sum_{i=1}^{k-1}\sum_{\tau_h}\Vert[M^k_{o_h,a_h,h}-M_{o_h,a_h,h}]b_{\tau_{h-1}}\Vert_1\times\pi^i(\tau_{h-1})\notag\\
 &\qquad\leq\sum_{i=1}^{k-1}\sum_{\tau_h}\Vert[M^k_{o_h,a_h,h}b^k_{\tau_{h-1}}-M_{o_h,a_h,h}b_{\tau_{h-1}}]\Vert_1\times\pi^i(\tau_{h-1})\notag\\
 &\qquad\quad+\sum_{i=1}^{k-1}\sum_{\tau_h}\Vert M^k_{o_h,a_h,h}[b^k_{\tau_{h-1}}-b_{\tau_{h-1}}]\Vert_1\times\pi^i(\tau_{h-1})\label{eq:estimation 5}.
\end{align}

For the first term in \eqref{eq:estimation 5}, we have
\begin{align}
\label{eq:estimation 8}
&\sum_{i=1}^{k-1}\sum_{\tau_h}\Vert[M^k_{o_h,a_h,h}b^k_{\tau_{h-1}}-M_{o_h,a_h,h}b_{\tau_{h-1}}]\Vert_1\times\pi^i(\tau_{h-1})\notag\\
&\qquad=\sum_{i=1}^{k-1}\sum_{\tau_h}\pi^{i}(\tau_{h-1})\cdot\Vert b^k_{\tau_h}-b_{\tau_h}\Vert_1\leq\BO(|\A|\sqrt{kH|\UA|\beta}),
\end{align}
where the second step is due to \eqref{eq:estimation 4}.

To bound the second term, we need to bound $\Vert M^k_{o,a,h}\Vert_{1,1}$ first, which is given in the following lemma:
\begin{lemma}
\label{lem:key}
For any $1\leq j_1\leq j_2\leq H-1$, trajectory $\tau_{j_1-1}$, policy $\pi$, $f\in\FC$ and $x\in\R^{|\UC_{j_1}|}$, we have
\begin{align*}
\sum_{\tau_{j_1:j_2}}\bigg\Vert\prod_{j=j_1}^{j_2}M_{o_j,a_j,j;f}x\bigg\Vert_1\times\pi(\tau_{j_1:j_2}|\tau_{j_1-1})\leq\frac{|\UA|}{\alpha}\Vert x\Vert_1.
\end{align*}%\jnote{this is the lemma that is very different from standard pomdp? we should say something about novelty etc here }
\end{lemma} 
%\wen{yeah, we should say that this is where the regularity condition in the assumption comes in, and say that this allows us to plug in $K K^\dagger$ in the middle.}
%\wen{looks like below you only use this lemma w/ $j_1 = j_2$?}

The proof of the above lemma uses the regularity condition in Assumption~\ref{ass:function} and Lemma~\ref{lem:redundancy}.  Naively, the product $\prod_{j=j_1}^{j_2} M_{o_j, a_j, j ;f}$ may indicate that the norm may grow exponentially. However, the condition that $M_{o,a, h;f}$'s row span belongs to the column span of $K_{h-1;f}$ (which is dereived from Lemma~\ref{lem:redundancy}) and the fact that $K_{h-1;f}^\dagger$ exists, we have:
\begin{align*}
\prod_{j=j_1}^{j_2} M_{o_j, a_j, j ;f} x = \prod_{j=j_1}^{j_2} M_{o_j, a_j, j ;f} K_{j_1-1;f} K_{j_1-1;f}^\dagger x
\end{align*}    Thus, we can bound $\| \prod_{j=j_1}^{j_2} M_{o_j, a_j, j ;f} K_{j_1-1;f} e_l\|_1$ by using the fact that $K_{j_1-1;f}e_l$ is a predictive state $q_{\tau^l_{j_1-1;f};f}$ corresponding to one of the minimum core histories $\tau^l_{j_1-1;f}$, and $\prod_{j=j_1}^{j_2} M_{o_j, a_j, j ;f} \cdot q_{\tau^l_{j_1-1;f};f}\times\pi(\tau_{j_1:j_2}|\tau_{j_1-1}) =  [ \BP(u|\tau^l_{j_1-1;f},\tau_{j_1:j_2}) \BP^{\pi_{\tau_{j_1-1}}}( \tau_{j_1:j_2} | \tau^l_{j_1-1;f} )]_{u\in \UC_{j_2+1}}$ where $\pi_{\tau_{j_1-1}}$ denote the policy $\pi(\cdot|\tau_{j_1-1})$. Note that the proof of the above lemma differs from the one in POMDPs since here we leverage the concept of minimum core histories and the core matrix which are unique to PSRs. The details are deferred to Appendix~\ref{proof:key}.  %\wen{i added the above, please double check}\whz{Yeah they look great!}

Therefore, using Lemma~\ref{lem:key} with $\pi=\pi^i_{1:h-1}\circ\unif(
\A)$, we have
\begin{align}
&\sum_{i=1}^{k-1}\sum_{\tau_h}\Vert M^k_{o_h,a_h,h}[b^k_{\tau_{h-1}}-b_{\tau_{h-1}}]\Vert_1\times\pi^i(\tau_{h-1})\notag\\
&\leq\frac{|\A||\UA|}{\alpha}\sum_{i=1}^{k-1}\sum_{\tau_{h-1}}\pi^{i}(\tau_{h-1})\cdot\Vert b^k_{\tau_{h-1}}-b_{\tau_{h-1}}\Vert_1\leq\BO(|\A|^2\sqrt{kH|\UA|^3\beta}/\alpha),\label{eq:estimation 7}
\end{align}
where the last step comes from \eqref{eq:estimation 3}.

Combining \eqref{eq:estimation 8} and \eqref{eq:estimation 7}, we have for all $k\in[K],h\in[H-1]$,
\begin{align}
\label{eq:estimation 9}
\sum_{i=1}^{k-1}\sum_{\tau_h}\Vert[M^k_{o_h,a_h,h}-M_{o_h,a_h,h}]b_{\tau_{h-1}}\Vert_1\times\pi^i(\tau_{h-1})\leq\BO(|\A|^2\sqrt{kH|\UA|^3\beta}/\alpha).
\end{align}

\subsection{Step 4: Connect Step 2 and Step 3} 
Recall that in Step 2 we want to bound
\begin{align*}
\sum_{k=1}^K\bigg(\sum_{h=1}^{H-1}\sum_{\tau_h}\Vert[M^k_{o_h,a_h,h}-M_{o_h,a_h,h}]b_{\tau_{h-1}}\Vert_1\times\pi^k(\tau_h)+\Vert q^k_0-q_0\Vert_1\bigg).
\end{align*} 

First, for the second term, we can bound via \eqref{eq:estimation 6}:
\begin{align}
\label{eq:connect 1}
\sum_{k=1}^K\Vert q^k_0-q_0\Vert_1\leq\BO(\sqrt{HK|\UA|\beta}).
\end{align}

Now we only need to bound the first term. Notice that in \eqref{eq:estimation 9} we have bounded this cumulative estimation error weighted by $\pi^i(\tau_{h-1})$ rather than $\pi^k(\tau_{h-1})$. Here we introduce the following lemma from \cite{liu2022partially} to bridge these two summations with different weights:
\begin{lemma}[{\cite[Proposition 22]{liu2022partially}}]
\label{lem:connect}
Suppose $\{x_{k,i}\}_{(k,i)\in[K]\times[n_1]},\{w_{k,j}\}_{(k,j)\in[K]\times[n_2]}\in\R^d$ satisfy for all $k\in[K]$ 
\begin{itemize}
	\item $\sum_{t=1}^{k-1}\sum_{i=1}^{n_1}\sum_{j=1}^{n_2}|w^{\top}_{k,j}x_{t,i}|\leq\gamma_k$,
	\item $\sum_{i=1}^{n_1}\Vert x_{k,i}\Vert_2\leq R_x$,
	\item $\sum_{i=1}^{n_2}\Vert w_{k,i}\Vert_2\leq R_w$.
\end{itemize}
Then we have for all $k\in[K]$:
\begin{align*}
\sum_{t=1}^k\sum_{i=1}^{n_1}\sum_{j=1}^{n_2}|w_{t,j}^{\top}x_{t,i}|=\BO\Big(d\Big(R_wR_x+\max_{t\leq k}\gamma_t\Big)\log^2(Kn_1)\Big).
\end{align*}
\end{lemma}

To apply Lemma~\ref{lem:connect}, for any fixed $h\in[H-1]$, we rewrite \eqref{eq:estimation 9} in the following way:
\begin{align}
&\sum_{t=1}^{k-1}\sum_{u=1}^{|\UC_{h+1}|}\sum_{o,a}\sum_{\tau_{h-1}}\big|\big[\big(M^k_{o,a,h}-M_{o,a,h}\big)K_{h-1}\big]_uK_{h-1}^{\dagger}b_{\tau_{h-1}}\times\pi^t(\tau_{h-1})\big|\notag\\
&\qquad\leq\BO(|\A|^2\sqrt{kH|\UA|^3\beta}/\alpha),\label{eq:connect 2}
\end{align}
where $X_u$ is the $u$-th row of the matrix $X$. Here we utilize the fact that $b_{\tau_{h-1}}\times\pi^t(\tau_{h-1})=(\BP[u|\tau_{h-1}]\BP^{\pi^t}[\tau_{h-1}])_{u\in\UC_h}$ belongs to the column space of $K_{h-1}$ due to the definition of core history. 

Then for any $t\in[K],u\in\UC_{h+1},o\in\BO,a\in\A$, we let $w_{t,u,o,a}$ denote $\big[\big(M^t_{o,a,h}-M_{o,a,h}\big)K_{h-1}\big]_u$ and $x_{t,\tau_{h-1}}$ denote $K_{h-1}^{\dagger}b_{\tau_{h-1}}\times\pi^t(\tau_{h-1})$, then \eqref{eq:connect 2} can be written as for any $k\in[K]$
\begin{align}
\label{eq:connect 5}
\sum_{t=1}^{k-1}\sum_{u\in\UC_{h+1},o\in\BO,a\in\A}\sum_{\tau_{h-1}}|w^{\top}_{k,u,o,a}x_{t,\tau_{h-1}}|\leq\BO(|\A|^2\sqrt{kH|\UA|^3\beta}/\alpha).
\end{align}

Now we only need to bound $\sum_{\tau_{h-1}}\Vert x_{k,\tau_{h-1}}\Vert_2$ and $\sum_{u\in\UC_{h+1},o\in\BO,a\in\A}\Vert w_{k,u,o,a}\Vert_2$. For $\sum_{\tau_{h-1}}\Vert x_{k,\tau_{h-1}}\Vert_2$, we have
\begin{align*}
&\sum_{\tau_{h-1}}\Vert x_{k,\tau_{h-1}}\Vert_2=\sum_{\tau_{h-1}}\Vert K_{h-1}^{\dagger}b_{\tau_{h-1}}\times\pi^k(\tau_{h-1})\Vert_2=\sum_{\tau_{h-1}}\Vert K_{h-1}^{\dagger}[\BP(u|\tau_{h-1})\BP^{\pi^k}(\tau_{h-1})]_{u\in\UC_h}\Vert_2\\
&\qquad=\sum_{\tau_{h-1}}\Vert v_{\tau_{h-1}}\Vert_2 \BP^{\pi^k}(\tau_{h-1})\leq\max_{\tau_{h-1}}\Vert v_{\tau_{h-1}}\Vert_2,
\end{align*}
where the third step comes from the definition of core matrix \eqref{eq:core history}.

Notice that we have $K_{h-1}v_{\tau_{h-1}}=[\BP(u|\tau_{h-1})]_{u\in\UC_h}$ for any $\tau_{h-1}$ and $\Vert K^{\dagger}_{h-1}\Vert_{1\mapsto1}\leq1/\alpha$, which implies
\begin{align*}
&\Vert v_{\tau_{h-1}}\Vert_2\leq\Vert v_{\tau_{h-1}}\Vert_1\leq\Vert K^{\dagger}_{h-1}[\BP(u|\tau_{h-1})]_{u\in\UC_h}\Vert_1\leq\frac{1}{\alpha}\Vert [\BP(u|\tau_{h-1})]_{u\in\UC_h}\Vert_1\leq\frac{|\UA|}{\alpha}.
\end{align*}
Therefore we have for all $k\in[K]$,
\begin{align}
\label{eq:connect 3}
\sum_{\tau_{h-1}}\Vert x_{k,\tau_{h-1}}\Vert_2\leq\frac{|\UA|}{\alpha}.
\end{align}

For $\sum_{u\in\UC_{h+1},o\in\BO,a\in\A}\Vert w_{k,u,o,a}\Vert_2$, we have
\begin{align}
&\sum_{u\in\UC_{h+1},o\in\BO,a\in\A}\Vert w_{k,u,o,a}\Vert_2\leq\sum_{u\in\UC_{h+1},o\in\BO,a\in\A}\Vert w_{k,u,o,a}\Vert_1\notag\\
&\qquad=\sum_{o\in\BO,a\in\A}\sum_{l=1}^{\dhj}\Vert((M^k_{o,a,h}-M_{o,a,h})K_{h-1}e_l\Vert_1\notag\\
&\qquad\leq\frac{2|\A||\UA|}{\alpha}\sum_{l=1}^{\dhj}\Vert K_{h-1}e_l\Vert_1\leq\frac{2|\A||\UA|^2\dps}{\alpha},\label{eq:connect 4}
\end{align}
where the third step utilizes Lemma~\ref{lem:key} with uniform policy and the last step utilizes the fact $\dhj\leq\dps$ and $\Vert K_{h-1}e_l\Vert_1=\Vert q_{\tau^l_{h-1}}\Vert_1\leq|\UA|$.

Invoking Lemma~\ref{lem:connect} with \eqref{eq:connect 3},\eqref{eq:connect 4},\eqref{eq:connect 5}, we can obtain for all $k\in[K],h\in[H-1]$,
\begin{align}
&\sum_{i=1}^{k}\sum_{\tau_h}\Vert[M^k_{o_h,a_h,h}-M_{o_h,a_h,h}]b_{\tau_{h-1}}\Vert_1\times\pi^k(\tau_{h-1})\notag\\
&\qquad\leq\BO(\dhj|\UA|^3|\A|^2\dps H^{\frac{3}{2}}k^{\frac{1}{2}}/\alpha^2\cdot\log(KH\ncov|\BO||\A|/\delta)).\label{eq:connect 6}
\end{align}

Substituing \eqref{eq:connect 1},\eqref{eq:connect 6} into \eqref{eq:difference 2}, we have
\begin{align*}
&\sum_{k=1}^{K}\big(V_{f^k}^{\pi^k}-V^{\pi^k}\big)\leq\BO(\dps^2H^{\frac{7}{2}}|\UA|^4|\A|^2K^{\frac{1}{2}}\alpha^{-3}\cdot\log(KH\ncov|\BO||\A|/\delta)).
\end{align*}

Combining the above result with Step 1, we have
\begin{align*}
&\sum_{k=1}^{K}\big(V^*-V^{\pi^k}\big)\leq\BO(\dps^2H^{\frac{7}{2}}|\UA|^4|\A|^2K^{\frac{1}{2}}\alpha^{-3}\cdot\log(KH\ncov|\BO||\A|/\delta)).
\end{align*}
This concludes our proof.

\section{Examples} \label{sec:example}

%%%%\jnote{Check whether HSE PSR is actually linear POMDP}
In this section, we illustrate the sample complexity of \mainalg to learn tabular PSRs and several POMDPs and compare the results with existing algorithms.  We consider two types of POMDPs: $m$-step weakly-revealing POMDPs and $m$-step decodable POMDPs. Assumptions like weakly-revealing property and decodability allow us to identify core test sets. Additionally, by assuming three types of structures on latent dynamics, we can further identify the minimum core test size and the bracketing number of models. This is summarized in Table~\ref{tab:example} and its proof is deferred to Appendix~\ref{sec:example appendix} and \ref{sec:bracket}.

\begin{table}[!t]
	\begin{center}
		\begin{tabular}{|c|c|c|c|}
			\hline
			Model & Core test set $\UC_h$ & $\dps$ & $\log\ncove$\\
			\hline
			 tabular POMDPs & \multirow{5}{*}{$(\BO\times\A)^{m-1}\times\BO$} & $\leq|\SC|$ & $\poly(|\BO|,|\A|,|\SC|,H,\log(1/\eps))$\\
			\cline{1-1}\cline{3-4}
            \multirow{3}{*}{low-rank POMDPs} & & \multirow{3}{*}{$\leq \dtran$} & $\log\nphi(\epsl/\dtran)$\\
            & & & $+\log\npsi(\epsl/\dtran)$\\
            & & & $+\log\nmo(\epsl)+\log\nmu(\epsl)$\\
            \cline{1-1}\cline{3-4}
            linear POMDPs & & $\leq \dlin$ & $\poly(\dlin,H,\log(|\BO||\A|/\eps))$\\
            \hline
		\end{tabular}
		\caption{Core test sets, minimum core test size and bracket number for POMDP models. Here all the POMDPs we consider are $m$-step weakly-revealing or $m$-step decodable. The exact function classes $\FC$ we use are elaborated in the following discussion.}\label{tab:example}
	\end{center}
\end{table}

\subsection{Tabular PSRs}
Notice that in Corollary~\ref{cor:main} the log bracket number $\log\ncov$ is abstract. Here we consider tabular PSRs as a speical case to provide an intuition how large the bracket number will be in general. In tabular PSRs we directly use $\{M_{o,a,h},q_{0}\}_{o\in\BO,a\in\A,h\in[H-1]}$ as the parameters of $\FC$ and assume for all $f\in\FC$ we have
\begin{align*}
	\max_{o\in\BO, a\in\A,h\in[H-1],u\in\UC_{h+1}}\Vert m_{(o,a,u),h;f}\Vert_{\infty}\leq1,\Vert q_{0;f}\Vert_{\infty}\leq 1 
\end{align*}
Then following the arguments in Appendix~\ref{sec:bracket}, we have 
\begin{align}
\label{eq:tabualr psr bracket}
	\log\N_{\FC}(\eps)\leq\BO(|\UC|^2|\BO||\A|H^2\log(H|\BO||\A||\UA||\UC|/(\alpha\eps))).
\end{align} 
By substituting the above results into Corollary~\ref{cor:main} and we have the following corollary which characterizes the sample complexity to learn tabular PSRs:
\begin{corollary}[Sample complexity for tabular PSRs]
	\label{cor:tabular psr}
	Execute \mainalg with $\beta=c\log(\ncov KH|\UA|/\delta)$ where $\ncov$ are specified in \eqref{eq:tabualr psr bracket}. Then under Assumption~\ref{ass:regular},\ref{ass:observe} and \ref{ass:function}, for any $\delta\in(0,1],\eps>0$, if we choose
	\begin{align*}
		K=1/\eps^2\times\poly(\dps,|\UA|,1/\alpha,|\UC|,|\BO|,|\A|,H,\log(1/\delta)),
	\end{align*} 
	then with probability at least $1-\delta$ we have
	\begin{align*}
		V^{\hpi}\geq V^*-\eps.
	\end{align*}
\end{corollary}
From Corollary~\ref{cor:tabular psr} we can see that \mainalg is capable of learning tabular PSRs efficiently, with sample complexity polynomial in all relevant parameters. Here, though we have a $\mathrm{poly}(|\UC|)$ dependency in  learning PSRs (since our model parameters $M_{a,o}$ have degree of freedom scaling in $\mathrm{poly}(|\UC|)$),  we will show that we would not incur $\mathrm{poly}(|\UC|)$ in the log $\epsilon$-bracket number when PSRs are POMDPs. This is because we can directly model the latent transition and omission distribution when we know it is a POMDP.

\subsection{$m$-Step Weakly-Revealing Tabular POMDPs }

We now focus on $m$-step weakly-revealing tabular POMDPs \citep{liu2022partially} defined as follows.

\begin{definition}[$m$-step weakly-revealing Tabular POMDPs]
Define the $m$-step emission matrix $\MO_{h,m}\in\R^{|\A|^{m-1}|\BO|^m\times|\SC|}$ for any $h\in[H-m+1]$ as follows:
\begin{align*}
	(\MO_{h,m})_{(\ba,\boo),s}:=\BP(o_{h:h+m-1}=\boo|s_h=s,a_{h:h+m-2}=\ba), \forall (\ba,\boo)\in\A^{m-1}\times\BO^m, s\in\SC.
\end{align*}
When $\mathrm{rank}(\MO_{h,m})=|\mathcal{S}|$, POMDPs are referred as $m$-step weakly-revealing POMDPs. 
\end{definition}

This assumption implies that the observations leak at least some information about the states so that we can learn the POMDPs efficiently. In this case we can choose $\UC_h$ to be the set of all $m$-step futures $(\BO\times\A)^{m-1}\times\BO$. For the function class $\FC$, we first let it model the parameters $\{\MT_{h;f},\MO_{h;f},\mu_{1;f}\}_{h\in[H]}$ directly, lift weakly-revealing POMDPs to PSR formulation, and then pre-process it to satisfy Assumption~\ref{ass:function}. The corresponding $\eps$-bracket number $\ncove$ is shown in Table~\ref{tab:example}. Besides, since now the core test set is $(\BO\times\A)^{m-1}\times\BO$, we let $m_{(o_{H-m+1:H},a_{H-m+1:H-1}),H-m+1;f} = e_{(o_{H-m+1:H},a_{H-m+1:H-1}),H-m+1}$, i.e., it is a one-hot vector which indexes the future $(o_{H-m+1:H},a_{H-m+1:H-1})$ in $\UC_{H-m+1}$. Then from Lemma~\ref{lem:product} we know for any trajecotry $\tau_H$,
\begin{align*}
\BP_f^{\pi}(\tau_H)=e_{(o_{H-m+1:H},a_{H-m+1:H-1}),H-m+1}\cdot\prod_{l=1}^{H-m}M_{o_l,a_l,l}q_0\times\pi(\tau_H).
\end{align*}

Note that here $\UC_H$ does not contain the observation space. Nevertheless, we can replace \eqref{eq:difference 1} with
\begin{align*} 
	V_{f^k}^{\pi^k}-V^{\pi^k}\leq H\sum_{\tau_{H-m}}\bigg\Vert\prod_{h=1}^{H-m}M^k_{o_h,a_h,h}\cdot q^k_0-\prod_{h=1}^{H-m}M_{o_h,a_h,h}\cdot q_0\bigg\Vert_1\times\pi^k(\tau_{H-m}),
\end{align*}
and follow the same proof to show that Theorem~\ref{thm:main} and thus Corollary~\ref{cor:main} still hold. Therefore, substituing the results in Table~\ref{tab:example} into Corollary~\ref{cor:main}, we can obtain the sample complexity for learning tabular POMDPs, which is shown in the following corollary:

\begin{corollary}[Sample complexity for $m$-step weakly-revealing tabular POMDPs]
	\label{cor:tabular}
	Suppose the POMDP is $m$-step weakly-revealing and we execute \mainalg with $\beta=c\log(\ncov KH|\UA|/\delta)$ up to the step $H-m$ where $\UC_h$ and $\ncov$ are specified in Table~\ref{tab:example}. Then under Assumption~\ref{ass:regular},\ref{ass:function}, for any $\delta\in(0,1],\eps>0$, if we choose
	\begin{align*}
	K=1/\eps^2\times\poly(\dps,|\A|^m,1/\alpha,|\BO|,|\SC|,H,\log(1/\delta)),
	\end{align*} 
 then with probability at least $1-\delta$ we have
	\begin{align*}
		V^{\hpi}\geq V^*-\eps.
	\end{align*}
\end{corollary}

By executing \mainalg to  the step $H-m$ we mean that when collecting trajectories, we only execute $\pi^{k,u_{a,h+1},h}$ and collect $\tau_{H}^{k,u_{a,h+1},h}$ for $h\in[H-m]^+,u_{a,h+1}\in\UAhp$. Since we have $\dps\leq|\SC|$, we can obtain that the sample compleixty will not be larger than $\poly(|\SC|,H,|\A|^m,1/\alpha,1/\eps,|\BO|,\log(1/\delta))$ from Corollary~\ref{cor:tabular}. This indicates that \mainalg is able to achieve polynomial sample complexity for $m$-step weakly-revealing tabular POMDPs.

\paragraph{Comparison with \cite{liu2022partially}.} In $m$-step weakly revealing tabular POMDPs, \mainalg is similar to the algorithm OMLE proposed in \cite{liu2022partially} and their analysis leads to a sample complexity similar to Corollary~\ref{cor:tabular}. However, their algorithm has a pre-processing step on the emission matrix $\MO_{h,m}$ while we have a step to formulate POMDPs into PSRs for pre-processing, thus the algorithm is still different. Further, they assume an upper bound on $\Vert\MO_{h,m}^{\dagger}\Vert_{1\mapsto1}$ while we assume $\Vert K_{h}^{\dagger}\Vert_{1\mapsto1}\leq1/\alpha$ in Assumption~\ref{ass:regular}. For tabular POMDPs, our assumption is slightly stronger since we have $K_{h-1}=\MO_{h,m}\BS_h$ where $(\BS)_{s,\tau_{h-1}^l}=\BP(s|\tau_{h-1}^l)$ and thus $\sigma_{\min}(K_{h-1})\leq\dps\sigma_{\min}(\MO_{h,m})$. That said, the analysis and algorithm in \cite{liu2022partially} is specially tailored to $m$-step weakly reavling POMDPs and relies on the existence of latent states. In contrast, our algorithm and analysis can be applied to any PSR models including $m$-step decodable POMDPs.

\paragraph{Comparison with \cite{kwon2021rl}.} \cite{kwon2021rl} deals with latent MDPs but they require either proper initialization or other assumptions including Sufficient Tests, Sufficient Histories, strong separation of the MDPs and reachability of the states. In contrast, we show in Appendix~\ref{sec:example appendix} that LMDP with Sufficient Tests can be formulated into a $(l+1)$-step weakly-revealing POMDP, therefore \mainalg is capable of tackling LMDP with sample complexity $1/\eps^2\times\poly(M,|\SC|,|\A|^l,1/\alpha,H,\log(1/\delta))$ under only Sufficient Tests and Assumption~\ref{ass:regular}. In addition, the sample complexity in \cite[Theorem~3.5]{kwon2021rl} will scale with the initialization error while \mainalg circumvents such dependence completely.

\subsection{$m$-Step Weakly-Revealing Low-Rank POMDPs}

Next, we consider $m$-step weakly-revealing low-rank POMDPs. We first define low-rank POMDPs to be a special subclass of POMDPs. 

\begin{definition}[Low-rank POMDPs] 
Suppose the transition kernel $\MT_{h}$ has the following low-rank form for all $h\in[H]$:
\begin{align}
	\label{eq:low def}
	\MT_{h}(s'|s,a)=(\psi_{h}(s'))^{\top}\phi_h(s,a),
\end{align} 
where $\psi_h:\SC\mapsto\R^{\dtran}$ and $\phi_h:\SC\times\A\mapsto\R^{\dtran}$ are unknown feature vectors. Then, we call these POMDPs as low-rank POMDPs. 
\end{definition}
%%%%%\masa{The assumption on $\psi$ is a bit unstandard. We usually assume $\|\int \psi_h(s)g(s)\mathrm{d}(s)\|\leq \sqrt{d}$ for any $g:\mathcal{S}\to [0,1]$. At least in Flambe. }\whz{yeah. indeed here we don't need to mention the bounds since we are not using them directly.}

 %%%%This definition is more general than \cite{wang2022embed} since we do not assume $\phi_h(s,a)$ to be a probability distribution.\masa{i think it might be reperetive}
  The low-rank structure leads to a smaller minimum core test set size than general POMDPs since as proved in Appendix~\ref{sec:example appendix}, we can show $\dps\leq\dtran$. Weakly-revealing low-rank POMDPs are defined as weakly-revealing POMDPs that have this low-rank structure. 

For weakly-revealing low-rank POMDPs, we can still choose $\UC_h$ to be the set of all $m$-step futures $(\BO\times\A)^{m-1}\times\BO$ due to the weakly-revealing property. For the function class $\FC$, we let it model the feature vectors, emission matrix and initial state distribution, i.e., $\{\Phi_{f},\Psi_{f}, \MO_{f}, \mu_f\}_{f\in\FC}$ where $\Phi_f:\SC\times\A\times[H]\mapsto\R^{\dtran}$, $\Psi_f:\SC\times[H]\mapsto\R^{\dtran}$, $\MO_f:\SC\times\BO\times[H]\mapsto[0,1]$, $\mu_f:\SC\mapsto[0,1]$ such that
\begin{align*}
\phi_{h;f}(s,a)=\Phi_f(s,a,h), \psi_{h;f}(s)=\Psi_{f}(s,h), \MO_{h;f}(o|s)=\MO_f(s,o,h), \mu_{1;f}(s)=\mu_f(s).
\end{align*}
Denote the $\ell_{\infty}$-norm covering number of $\Phi_f,\Psi_f,\MO_f,\mu_f$ by $\nphi(\eps),\npsi(\eps),\nmo(\eps),\nmu(\eps)$. Then we have
\begin{align*}
\log\N_{\FC}(\eps)\leq \log\nphi(\epsl/\dtran)+\log\npsi(\epsl/\dtran)+\log\nmo(\epsl)+\log\nmu(\epsl),
\end{align*}
where $\epsl:=\BO(\eps/(|\BO|^{H+2}|\A|^H))$. The proof is deferred to Appendix~\ref{sec:bracket}.

Using the above models and formulating POMDPS into PSRS for the pre-processing step to satisfy Assumption~\ref{ass:function}, we can run \mainalg and the sample complexity for learning low-rank POMDPs is as follows:
\begin{corollary}[Sample complexity for $m$-step weakly-revealing low-rank POMDPs] 
	\label{cor:low rank}
	Suppose low-rank POMDPs are $m$-step weakly-revealing, %or $m$-step decoadble and
	and we execute \mainalg with $\beta=c\log(\ncov KH|\UA|/\delta)$ up to the step $H-m$ where $\UC_h$ and $\ncov$ are specified in Table~\ref{tab:example}. Then under Assumption~\ref{ass:regular} and \ref{ass:function}, for any $\delta\in(0,1],\eps>0$, if we choose
	\begin{align*}
		K=1/\eps^2\times\poly(\dtran,|\A|^m,1/\alpha,&\log\nphi(\epsl/\dtran),\log\npsi(\epsl/\dtran),\\
		&\log\nmo(\epsl),\log\nmu(\epsl),H,\log|\BO|,\log(1/\delta)),
	\end{align*} 
	then with probability at least $1-\delta$ we have
	\begin{align*}
		V^{\hpi}\geq V^*-\eps.
	\end{align*}
\end{corollary}

\paragraph{Comparison with \cite{wang2022embed}.} In Corollary~\ref{cor:low rank}, we do not specify the function class and keep the bracket number to facilitate the comparison with \cite{wang2022embed}. \cite{wang2022embed} also tackles the online learning problem of $m$-step weakly-revealing low-rank POMDPs and our sample complexity only has an additional $\log|\BO|$ factor compared to theirs. However, the model they have considered is less general than ours in the sense that they require the feature vectors $\phi_h(s,a)$ to be a $\dtran$-dimensional probability distribution to guarantee the existence of some bottleneck variables. Besides, their analysis depends on some possibly complicated assumptions to recover the bottleneck variable like “\textit{past sufficiency}" and “\textit{future sufficiency}". In contrast, \mainalg only requires $\Vert K_h^{\dagger}\Vert_{1\mapsto1}$ to be upper bounded and does not assume the existence of bottleneck variables.   

\paragraph{Comparsion with \cite{uehara2022provably}.} They show favorable sample complexity results in weakly-revealing low-rank POMDPs. However, their sample complexity results are quasi-polynomial. On the other hand, our results are polynomial while we have an additional $\log|\BO|$ factor.

\subsection{$m$-Step Weakly-Revealing Linear POMDPs}

In low-rank POMDPs the bracket number is still somehow abstract because we do not specify the function class $\{\Phi_{f},\Psi_{f}, \MO_{f}, \mu_f\}$. Next we consider linear POMDPs and illustrate a more concrete result. Here we assume that linear POMDPs possess a linear structure in both the transition kernel and emission matrix. More formally, we can generalize the linear MDPs in \cite{yang2020reinforcement} and define linear POMDPs as follows:
\begin{definition}[Linear POMDPs]
	A POMDP is linear with respect to the given feature vectors $\{\phi_h(s,a)\in\R^{d_1},\psi_h(s)\in\R^{d_2},\bphi_h(s)\in\R^{d_3},\bpsi_h(o)\in\R^{d_4},\hphi(s)\in\R^{d_5}\}_{s\in\SC,a\in\A,o\in\BO,h\in[H]}$ where $\Vert\phi_h(s,a)\Vert_{\infty}\leq1,\Vert\psi_h(s)\Vert_{\infty}\leq1,\Vert\bphi_h(s,a)\Vert_{\infty}\leq1,\Vert\bphi_h(o)\Vert_{\infty}\leq1,\Vert\hphi(s)\Vert_{\infty}\leq1$ for all $s\in\SC,a\in\A,o\in\BO,h\in[H]$ if there exists a set of matrices $\{B^*_{h,1}\in\R^{d_2\times d_1},B^*_{h,2}\in\R^{d_4\times d_3}\}_{h\in[H]}$ where $\Vert B^*_{h,1}\Vert_{\infty,\infty}\leq1,\Vert B^*_{h,2}\Vert_{\infty,\infty}\leq1$ for all $h\in[H]$ and $\theta^*\in\R^{d_5}$ where $\Vert\theta^*\Vert_{\infty}\leq1$ such that for any $s,s'\in\SC, a\in\A, o\in\A, h\in[H]$ we have
	\begin{align*}
		\MT_h(s'|s,a)=(\psi_{h}(s'))^{\top}B^*_{h,1}\phi_h(s,a), \MO(o|s)=(\bpsi_{h}(o))^{\top}B^*_{h,2}\bphi_h(s),\mu_1(s)=(\theta^*)^{\top}\hphi(s).
	\end{align*} 
\end{definition}
Denote $\dlin=\max\{d_1,d_2,d_3,d_4,d_5\}$. Notice that since $\MT_h(s'|s,a)=(\psi_{h}(s'))^{\top}B^*_{h,1}\phi_h(s,a)$, linear POMDPs are also low-rank POMDPs with dimension $\min\{d_1,d_2\}$ and thus for linear POMDPs we have 
\begin{align*}
%	\label{eq:linear 1}
	\dps\leq \dlin.
\end{align*}

We further define the function class to be $\{f=(B_{h,1}\in\R^{d_2\times d_1},B_{h,2}\in\R^{d_4\times d_3},\theta\in\R^{d_5}):\Vert B_{h,1}\Vert_{\infty,\infty}\leq1,\Vert B_{h,2}\Vert_{\infty,\infty}\leq1,\Vert\theta\Vert_{\infty}\leq1, \forall h\in[H]\}$ such that for any 
$o\in\BO,a\in\A,h\in[H-m]$
\begin{align*}
\MT_{h;f}(s'|s,a)=(\psi_{h}(s'))^{\top}B_{h,1}\phi_h(s,a),\MO_{h;f}(o|s)=(\bpsi_{h}(o))^{\top}B_{h,2}\bphi_h(s),\mu_{1;f}(s)=\theta^{\top}\hphi(s).
\end{align*}
This enables us to bound $\ncove$ as in Table~\ref{tab:example}. 

Finally, we assume $m$-step weakly-revealing property. In this case, we can still choose the same $\UC_h$ as in tabular POMDPs. 
Using the above models and formulating POMDPS into PSRS for the pre-processing step to satisfy Assumption~\ref{ass:function}, we can run \mainalg. The sample complexity for learning linear POMDPs will scale with $\dlin$ rather than $\mathrm{poly}(|\BO|,|\SC|)$ as follows.

\begin{corollary}[Sample complexity of $m$-step weakly-revealing linear POMDPs]
	\label{cor:linear}
	Suppose the linear POMDP is $m$-step weakly-revealing and we execute \mainalg with $\beta=c\log(\ncov KH|\UA|/\delta)$ up to the step $H-m$ where $\UC_h$ and $\ncov$ are specified in Table~\ref{tab:example}. Then under Assumption~\ref{ass:regular} and \ref{ass:function}, for any $\delta\in(0,1],\eps>0$, if we choose
	\begin{align*}
%		\label{eq:cor tabular} 
		K=1/\eps^2\times\poly(\dlin,|\A|^m,1/\alpha,H,\log|\BO|,\log(1/\delta)),
	\end{align*} 
	then with probability at least $1-\delta$ we have
	\begin{align*}
		V^{\hpi}\geq V^*-\eps.
	\end{align*}
\end{corollary}

\paragraph{Comparison with \cite{cai2022sample}.} From Corollary~\ref{cor:linear}, we can see that the linear structure helps us circumvent the polynomial scaling with $|\BO|$ and $|\SC|$. \cite{cai2022sample} also discusses linear POMDPs and achieves a similar polynomial sample complexity. However, in their model, they not only assume the transition and emission are linear, but also impose a linear structure on the state distribution conditioned on future observations such as \citet[Assumption 2.2]{cai2022sample}.

\subsection{$m$-Step Decodable Tabular/Low-Rank/Linear POMDPs}

Next, we instantiate our result on $m$-step decodable POMDPs \citep{efroni2022provable} defined as follows. 

\begin{definition}[$m$-step decodable POMDPs]
There exist unknown decoders $\{\phid\}_{m\leq h\leq H}$ such that for every reachable trajctory $\tau_H$, we have $s_h=\phid(z_h)$ for all $m\leq h\leq H$ where $z_h=((o,a)_{h-m+1:h-1},o_h)$. 
\end{definition}

This definition says that we can decode the current state with $m$-step history. When $m=1$, they are reduced to block MDPs \citep{du2019provably}. Surprisingly, $m$-step decodable POMDPs can be formulated as PSRs where core tests are $m$-step futures. Intuitively, this is proved by the observation that $m$-step futures can decode the latent state $m$-step ahead, i.e., $s_{m+h}$ by treating “histories" in the definition as “futures". In Appendix~\ref{sec:example appendix} we have a more detailed discussion. Thus, further noting the PSR rank is $|\mathcal{S}|$ in the tabular case and $d_{\mathrm{trans}}$ in low-rank POMDPs, we have the following theorem:

\begin{corollary}[Sample comlexity for $m$-step decodable POMDPs]~
\label{cor:decodable}
\begin{itemize}
	\item In $m$-step decodable tabular POMDPs, the same statement in Corollary~\ref{cor:tabular} holds. 
	\item In $m$-step decodable low-rank POMDPs, the same statement in Corollary~\ref{cor:low rank} holds. 
\item In $m$-step decodable linear POMDPs, the same statement in Corollary~\ref{cor:linear} holds. 
\end{itemize}
\end{corollary}

\paragraph{Comparison with \cite{efroni2022provable}.} \cite{efroni2022provable} works on $m$-step decodable tabular POMDPs and show sample complexity polynomial in $|\SC|,H,|\A|^m,1/\eps,\log(1/\delta)$ and log covering number of a value function class. They also provide a result on $m$-step decodable low-rank POMDPs where the sample complexity scales with $\dtran$ rather than $|\SC|$. 

However, there are some differences between their results and Corollary~\ref{cor:decodable}. First, the log covering number of the value function class in their results will typically scale with $\mathrm{poly}(|\BO|^m)$. Our results, on the other hand, only scale with $\mathrm{poly}(|\BO|)$ since the log bracket number of our function classes only scales with $\mathrm{poly}(|\BO|)$. Secondly, the analysis in \cite{efroni2022provable} does not require the regularity-type assumption (Assumption~\ref{ass:regular}). This is because their algorithm is tailored to $m$-step decodable POMDPs. The lower bound in Theorem~\ref{thm:lower bound} has shown that the scaling with the regularity parameter $1/\alpha$ is inevitable in PSRs, highlighting the necessity of such regularity in general.

\section{Conclusion} 

We consider PAC learning in PSRs that represent states as a vector of prediction about future events conditioned on histories. We propose \mainalg and show polynomial sample complexities when we compete with the globally optimal policy. Our work is the first work attending this goal. Since PSRs are more general than POMDPs, we instantiate our result to several concrete POMDPs such as $m$-step weak-revealing POMDPs, $m$-step decodable POMDPs, POMDPs with latent low-lank transition, and POMDPs with linear emission and latent transition. Notably, our work is the first work that simultaneously achieves polynomial sample complexities in $m$-step weak-revealing POMDPs and $m$-step decodable POMDPs. In POMDPs with latent low-lank transition and POMDPs with linear emission and latent transition, our result is still new in that we show some of the assumptions in existing literature \citep{cai2022sample,wang2022embed} can be replaced with possibly milder assumptions, i.e., regularity assumption. %\whz{\checkmark.}

\clearpage 
\bibliographystyle{apalike}
\bibliography{ref.bib,ref-pomdp.bib}

\clearpage 
\appendix
\section{Expressivity of PSRs}
\label{sec:rank}
In this section, we will construct a sequential decision making process to illustrate the superior expressivity of PSRs with respect to POMDPs. In short, we will show that if we formulate the process into a POMDP, the number of latent states we need can be exponentially larger the core test set size in PSRs. The construction leverages existing results in perfect matching polytope and largely follows the arguments in \cite{agarwal2020flambe}.

First, let $n$ be even and $K_n$ be the complete graph on $n$ vertices. Consider a vector $x\in\R^{\binom{n}{2}}$ that associates a weight to each edge and we denote its entry by $x_{u,v}$ where $u\neq v\in[n]$ are the vertices. Let $\bon_{\MC}\in\R^{\binom{n}{2}}$ denote the edge-indicator vector for a subset of edges $\MC$. Then \cite{edmonds1965maximum} shows that the convex hull of all edge-indicator vectors corresponding to a perfect match, which we also call the perfect matching polytope, can be expressed with a number of constraints as follows:
\begin{align*}
\mathcal{P}_n&:=\mathrm{conv}\bigg\{\bon_{\MC}\in\R^{\binom{n}{2}}|\MC \text{ is a perfect matching in } K_n\bigg\}\\
&=\bigg\{x\in\R^{\binom{n}{2}}:x\geq0;\forall v: \sum_{u}x_{u,v}=1;\forall U\subset[n] \text{ and $|U|$ is odd}:\sum_{v\notin U}\sum_{u\in U}x_{u,v}\geq1\bigg\}.
\end{align*}

There are $V:=n!/(2^{n/2}(n/2)!)$ vertices in $\mathcal{P}_n$ and the number of constraints is $C:=2^{\Omega(n)}$. We denote the vertices by $\{v_1,\cdots,v_V\}$ and the constraints by $c_1,\cdots,c_C$. We further add another dimension to $v_i(i\in[V])$ to account for the offsets in the constraints and obtain vectors $v'_i\in\R^{\binom{n}{2}+1}(i\in[V])$. Then we have $\langle c_i, v'_j\rangle\geq0$ for all $i\in[C],j\in[V]$. Now we can define the slack matrix for the polytope $\mathcal{P}_n$ to be $Z\in\R_{+}^{C\times V}$ where $Z_{i,j}=\langle c_i,v'_j\rangle$.

Notice that the rank of $Z$ is $\BO(n^2)$. However, since $\mathcal{P}_n$ has extension complexity $2^{\Omega(n)}$ \citep{rothvoss2017matching} and the extension complexity of a polytipe is the non-negative rank of its slack matrix \citep{fiorini2013combinatorial}, we know the non-negative rank of $Z$ is at least $2^{\Omega(n)}$.

Now we can construct our sequential decision making process. Suppose for the step $1\leq h\leq H-1$, the process behaves according to a POMDP with state space $\SC'$, action space $\A$, observation space $\BO$, initial state distribution $\mu_1$, emission matrix $\MO_{h}$ and transition kernel $\MT_{h}$. At step $h=H$ though, the one-step system dynamics $D_{H-1}$ is given by associating each pair $(o_{H-1},a_{H-1})$ with a constraint $c_i$ and each future test $t\in\BO$ (which is one-step observation now) with a vertex $v'_j$ such that for any history $\tau_{H-1}$ that ends with $(o_{H-1},a_{H-1})$ we have
\begin{align*}
\BP(t|\tau_{H-1})=\frac{\langle c_i, v'_j\rangle}{\sum_{k=1}^V\langle c_i,v'_{j}\rangle}.
\end{align*}

Now we fix a history $\tau_{H-2}$ with length $H-2$ and consider the matrix $\hat{D}_{H-1}\in\R^{|\BO|\times(|\BO||\A|)}$ where the rows are indexed by the test $t\in\BO$, the columns are indexed by the history $(\tau_{H-2},o,a)$ for all $o\in\BO,a\in\A$ and $(\hat{D}_{H-1})_{t,(\tau_{H-2},o,a)}=\BP(t|\tau_{H-2},o,a)$. Since the non-negative rank is preserved under positive diagonal rescaling \citep{cohen1993nonnegative}, we know the non-negative rank of $\hat{D}_{H-1}^{\top}$ is at least $2^{\Omega(n)}$. Then for the above sequential process, if we formulate it into a POMDP with state space $\SC$, then we have
\begin{align*}
\BP(t|\tau_{H-2},o,a)=\sum_{s_H\in\SC}\BP(t|s_H)\BP(s_H|\tau_{H-2},o,a).
\end{align*}  
Notice that for a row-stochatic matrix $\hat{D}_{H-1}^{\top}$, the non-negative rank is equal to the smallest number of factors we can use to write $\hat{D}_{H-1}^{\top}=RS$ where both $R,S$ are row-stochastic \citep{cohen1993nonnegative}. This implies that we must have $|\SC|$ not smaller than the non-negative rank of $\hat{D}_{H-1}^{\top}$, therefore we have
\begin{align*}
|\SC|\geq2^{\Omega(n)}.
\end{align*}

On the other hand, if we formulate the above process into a PSR, at step $h=H$, since the rank of $D_{H-1}$ is equal to the rank of $Z$, we know the rank of $D_{H-1}$ is not larger than $\BO(n^2)$, which implies that there exists a core test set $\UC_H$ whose size is not larger than $\BO(n^2)$. When $1\leq h\leq H-1$, notice that for any test $t=\{o_{h:H},a_{h:H-1}\}$ and history $\tau_{h-1}$ we have
\begin{align*}
\BP(t|\tau_{h-1})=\sum_{s_h\in\SC'}\BP(s_h|\tau_{h-1})(\BP(t_{h:H-1}|s_h)\BP(o_H|o_{H-1},a_{H-1})),
\end{align*}
where $t_{h:H-1}=\{o_{h:H-1},a_{h:H-2}\}$. Notice that $\BP(t_{h:H-1}|s_h)\BP(o_H|o_{H-1},a_{H-1})$ only depends on $t$ and $\BP(s_h|\tau_{h-1})$ only depends on $\tau_{h-1}$. This implies that there exists a core test set $\UC_h$ whose size is not larger than $|\SC'|$ for all $1\leq h\leq H-1$. Therefore, the core test set size of the PSR can be smaller than $\max\{\BO(n^2),|\SC'|\}$. This shows that PSRs can express this sequential decision making process exponentially more efificient than POMDPs.
\section{Examples of PSRs}
\label{sec:example appendix}
In this section we formulate $m$-step weakly-revealing POMDPs,  $m$-step decodable POMDPs and low rank POMDPs into PSRs and analyze their core test set and minimum core test set size. 

\subsection{$m$-Step Weakly-Revealing POMDPs}
Recall the definition of the $m$-step emission matrix $\MO_{h,m}\in\R^{|\A|^{m-1}|\BO|^m\times|\SC|}$ for any $h\in[H-m+1]$ is as follows:
\begin{align*}
(\MO_{h,m})_{(\ba,\boo),s}:=\BP(o_{h:h+m-1}=\boo|s_h=s,a_{h:h+m-2}=\ba), \forall (\ba,\boo)\in\A^{m-1}\times\BO^m, s\in\SC.
\end{align*}
Then $m$-step weakly revealing condition \citep{liu2022partially} means that the rank of $\MO_{h,m}$ is $|\SC|$ for all $h\in[H-m+1]$. From Lemma~\ref{lem:pomdp dps}, we know that $\dps\leq|\SC|$. In addition, the following lemma suggests that a general core test set for $m$-step weakly-revealing POMDPs is the set of all $m$-step futures:
\begin{lemma}
\label{lem:over}
	When $\MO_{h,m}$ is full rank for all $h\in[H-m+1]$, the POMDP is a PSR with the core test set $\UC_h=\BO\times(\A\times\BO)^{m-1}$ for all $h\in[H-m+1]$. 
\end{lemma}
\begin{proof}
	The proof is similar to 1-step weakly-revealing POMDPs. Consider any $h\in[H-m+1]$, let $q_{\tau_{h-1}}=[\BP(u|\tau_{h-1})]_{u\in\BO\times(\A\times\BO)^{m-1}}$. Then the belief state $\bs_{\tau_{h-1}}=[\BP(s_h|\tau_{h-1})]_{s_h\in\SC}$ can be expressed as:
	\begin{align*}
		\bs_{\tau_{h-1}}=\MO_{h,m}^{\dagger}q_{\tau_{h-1}}.
	\end{align*}
	Then we know for any test $t=(o_{h:h+W},a_{h:h+W-1})$, we know $\BP(t|\tau_{h-1})=m'_{t,h}\bs_{\tau_{h-1}}$ where
	\begin{align*}
		m'_{t,h}=\MO_{h+W}(o_{h+W}|\cdot)^{\top}\prod_{l=h}^{h+W-1}\MT_{l,a_l}\diag(\MO_l(o_l|\cdot)).
	\end{align*} 
	Recall that here $\MO_h(o|\cdot)\in\R^{|\SC|}$ is a vector whose $s$-th entry is $\MO_h(o|s)$ and $\mathbb{T}_{l,a_l}$ is a $|\mathcal{S}|\times |\mathcal{S}|$ matrix with entry $(\MT_{l,a_l})_{s',s}=\MT_l(s'|s,a_l)$.
	
	Therefore we have $\BP(t|\tau_{h-1})=\langle m_{t,h},q_{\tau_{h-1}}\rangle$ where $m_{t,h}=(m'_{t,h}\MO_{h,m}^{\dagger})^{\top}$. This indicates that $\UC_h=\BO\times(\A\times\BO)^{m-1}$ is a core test set for all $h\in[H-m+1]$.
\end{proof}

Notice that here we only show the core test set of  $m$-step weakly-revealing POMDPs up to step $H-m+1$. However, this is sufficient to charaterize the whole POMDP. From Lemma~\ref{lem:product} we know that for any trajectory $\tau_H$, $\BP^{\pi}(\tau_H)$ is one of the entries in $\prod_{l=1}^{H-m}M_{o_l,a_l,l}q_0\times\pi(\tau_H)$. Therefore, with parameters $\{M_{o,a,h},q_0\}_{o\in\BO,a\in\A,h\in[H-m]}$ (which only depends on $\UC_h$ where $h\in[H-m+1]$) we can recover the POMDPs easily.

\subsection{Latent MDPs}
Next we consider the Latent MDP (LMDP) model in \cite{kwon2021rl}. Supoose there are $M$ MDPs and each MDP $m$ is characterized by $(\SC,\A,\MT_{m,h}, R_{m,h},H,\mu_m)$ where $\SC$ is the common state space, $\A$ is the common action space, $\MT_{m,h}$ is the transition probability at step $h$ of MDP $m$, $R_{m,h}:\SC\times\A\times\{0,1\}\mapsto[0,1]$ is a probability measure for rewards at step $h$ of MDP $m$ that maps a state-action pair and a binary reward to a probability, $H$ is the horizon and $\mu_m$ is the initial state distribution of MDP $m$. At the start of every episode, one MDP $m\in[M]$ is randomly chosen with some probability $w_m$.

\cite[Theorem 3.1]{kwon2021rl} shows that with no further assumptions, learning an instance of the above LMDP requires at least $\Omega((|\SC||\A|)^M)$ episodes at worst. A number of assumptions are considered to circumvent this lower bound and one of them is called \textit{Sufficient Tests}. More specifically, for each step $h\in[H-l+1]$, consider all possible length-$l$ action-reward-state sequences $a_{h},r_{h},s_{h+1},\cdots,a_{h+l-1},r_{h+l-1},s_{h+l}$ and denote the set of all such sequences by $\T_h$. Then suppose that the successful probability of $\T_h$ under different MDPs given any $s_h\in\SC$ has rank $M$:
\begin{assumption}[{Sufficient Tests, \citep[Condition 1]{kwon2021rl}}]
\label{ass:suff test}
For any $h\in[H-l+1],s_h\in\SC$ and any $t=(a_{h}^t,r_{h}^t,s_{h+1}^t,\cdots s_{h+l}^t)\in\T_h$, we define 
\begin{align*}
	\BP_m(t|s_h):=\BP_m(r_{h}^t,s_{h+1}^t,\cdots s_{h+l}^t|s_h,\doi(a_{h}^t,\cdots,a_{h+l-1}^t)),
\end{align*}
where $\BP_m$ denotes the probability under MDP $m$. Let $L_{s_h}=[[\BP_1(t|s_h)]_{t\in\T_h},\cdots,[\BP_M(t|s_h)]_{t\in\T_h}]$, then $\sigma_{M}(L_{s_h})\geq\alpha$ for all $s_h\in\SC$ with some $\alpha>0$.
\end{assumption}

The following lemma indicates that LMDPs with Assumption~\ref{ass:suff test} can be formulated into an $(l+1)$-step weakly-revealing POMDP and thus a PSR with the core test set being all $(l+1)$-step futures:
\begin{lemma}
\label{lem:latent}
Under Assumption~\ref{ass:suff test}, the LMDP can be formulated into an $(l+1)$-step weakly-revealing POMDP. 
\end{lemma}
\begin{proof}
First notice that the LMDP can be formulated into a POMDP with state space $\OS=\SC\times\{0,1\}\times[M]$ and observation space $\BO=\SC\times\{0,1\}$. At each step $h$, the latent state $\os_h\in\OS$ is $(s_h,r_{h-1},I)$ where $s_h$ is the current observed state, $r_{h-1}$ is the reward of last step and $I$ is the index of the underlying MDP. On the other hand, the observation $o_h$ is $(s_h,r_{h-1})$. Then for any $h\in[H-l+1]$, any latent state $\os_h=(s_h,r_{h-1},I)$ and $(l+1)$-step test $\ot=(o_{h}^t,a_{h}^t,\cdots,o_{h+l}^t)$, we have
\begin{align*}
\BP(\ot|\os_h)=\bon(o_{h}^t=(s_h,r_{h-1}))\cdot\BP_{I}(t|s_h),
\end{align*}
where $t=(a_{h}^t,o_{h+1}^t,\cdots,o_{h+l}^t)$. Therefore, the $(l+1)$-step emission matrix can be written as follows:
\begin{align*}
\MO_{h,l+1}=
\begin{bmatrix}
L_{s_h^1} & 0 & 0 & 0 &\cdots&0\\
0 & L_{s_h^1} & 0 & 0 &\cdots&0\\
0 & 0 & L_{s_h^2} & 0 &\cdots&0\\
0 & 0 & 0 & L_{s_h^2} &\cdots&0\\
\vdots & \vdots & \vdots & \vdots & \ddots &\vdots \\
0 & 0 & 0 & 0 &\cdots&L_{s_h^{|\SC|}}
\end{bmatrix}.
\end{align*} 
Since the rank of $L_{s_h}$ is $M$ for any $s_h\in\SC$, we know the rank of $\MO_{h,l+1}$ is $2M|\SC|$ for all $h\in[H-l+1]$. This implies that the POMDP satisfies the $(l+1)$-step weakly-revealing condition.
\end{proof}

\subsection{$m$-Step Decodable POMDPs}
Recall the definition of $m$-step decodable POMDPs \citep{efroni2022provable} is that there exist unknown decoders $\{\phid\}_{m\leq h\leq H}$ such that for every reachable trajctory $\tau_H$, we have $s_h=\phid(z_h)$ for all $m\leq h\leq H$ where $z_h=((o,a)_{h-m+1:h-1},o_h)$. From Lemma~\ref{lem:pomdp dps}, we know that $\dps\leq|\SC|$. Further, the following lemma suggests that a general core test set for $m$-step decodable POMDPs is the set of all $m$-step futures:
\begin{lemma}
	\label{lem:decode}
	 A $m$-step decodable POMDP is a PSR with the core test set $\UC_h=\BO\times(\A\times\BO)^{m-1}$ for all $h\in[H-m+1]$. 
\end{lemma}
\begin{proof}
	Consider any $h\in[H-m+1]$, let $q_{\tau_{h-1}}=[\BP(u|\tau_{h-1})]_{u\in\BO\times(\A\times\BO)^{m-1}}$. Then for any test $t=(o_{h:h+W},a_{h:h+W-1})$, when $W\leq m-1$, we have for any length-$(m-1-W)$ action sequence $\ba$,
	\begin{align*}
	\BP(t|\tau_{h-1})=\sum_{u\in\UC_{t,\ba}}\BP(u|\tau_{h-1}),
	\end{align*} 
where $\UC_{t,\ba}$ denotes the set of all length-$m$ tests whose action sequence is $(a_{h:h+W-1},\ba)$ and the first $W+1$ observations are $o_{h:h+W}$. This implies $\BP(t|\tau_{h-1})=m_{t,h}^{\top}q_{\tau_{h-1}}$ where $m_{t,h}$ sets the entries corresponding to the tests in $\UC_{t,\ba}$ as 1 and the others as 0.

When $W>m-1$, we denote $t_{h:h+m-1}$ to be $(o_{h:h+m-1},a_{h:h+m-1})$ and $t_{h+m}$ to be $(o_{h+m:h+W},a_{h+m:h+W-1})$. Then we have
\begin{align*}
	\BP(t|\tau_{h-1})&=\BP(t_{h:h+m-1}|\tau_{h-1})\BP(o_{h+m:h+W}|(\tau_{h-1},t_{h:h+m-1});\doi(a_{h+m:h+W-1}))\\
	&=\BP(t_{h:h+m-1}|\tau_{h-1})\BP(o_{h+m:h+W}|\phim(z_{h+m-1}),a_{h+m-1};\doi(a_{h+m:h+W-1})).
\end{align*}
Notice that $\BP(o_{h+m:h+W}|\phim(z_{h+m-1}),a_{h+m-1};\doi(a_{h+m:h+W-1}))$ only depends on $t$, therefore we have $\BP(t|\tau_{h-1})=m_{t,h}^{\top}q_{\tau_{h-1}}$ where $m_{t,h}$ sets the entry corresponding to $(o_{h:h+m-1},a_{h:h+m-2})$ as $\BP(o_{h+m:h+W}|\phim(z_{h+m-1}),a_{h+m-1};\doi(a_{h+m:h+W-1}))$ and the others as 0. This concludes our proof.
\end{proof}
Similar to the discussion for $m$-step weakly-revealing POMDPs, it is suffcient to show the core test set of $m$-step decodable POMDPs up to step $H-m+1$.

\subsection{Low-Rank POMDPs}
Next we consider low-rank POMDPs. Recall that for low-rank POMDPs, the transition kernel $\MT_{h}$ has the following low-rank form for all $h\in[H]$:
\begin{align*}
\MT_{h}(s'|s,a)=(\psi_{h}(s'))^{\top}\phi_h(s,a),
\end{align*} 
where $\psi_h:\SC\mapsto\R^{\dtran}$ and $\phi_h:\SC\times\A\mapsto\R^{\dtran}$. The next lemma indicates that for low-rank POMDPs, the minimum core test set size will be not larger than $\dtran$, which can be potentially much smaller than $|\SC|$:
\begin{lemma}
\label{lem:low rank}
For any low-rank POMDP, its minimum core test set size will be not larger than $\dtran$.
\end{lemma}
\begin{proof}
First notice that we have for any test $t$ and history $\tau_{h}$:
\begin{align*}
\BP(t|\tau_h)=\langle [\BP(t|s_{h+1})]_{s_{h+1}\in\SC}, \bs_{\tau_h}\rangle.
\end{align*}
Besides, notice that from the low rank structure \eqref{eq:low def}, we have for any $s_{h+1}\in\SC$,
\begin{align*}
\BP(s_{h+1} | \tau_h) = \psi_h(s_{h+1})^{\top} \sum_{s_{h} \in \SC} \phi_h(s_h, a_h) \BP(s_h |  \tau_{h}).
\end{align*} 
This implies that
\begin{align*}
\BP(t | \tau_{h}) =\bigg( \sum_{s_{h+1}\in\SC} \BP(t | s_{h+1})\psi_h(s_{h+1}) \bigg)^{\top} \cdot\bigg( \sum_{s_{h}\in\SC} \phi_h(s_{h}, a_{h}) \BP(s_{h} | \tau_{h})   \bigg).
\end{align*}
This implies that the rank of the one-step system dynamics $D_h$ is not larger than $\dtran$ for all $h\in[H]$. Therefore we have $\dps\leq\dtran$. 
\end{proof}

\section{Proofs of PSR Dynamics}
\subsection{Proof of Lemma~\ref{lem:forward}}
\label{proof:lem forward}
To simplify writing, we denote the observation part of $t$ by $o(t)$ and the action part by $a(t)$. Then from Bayes rule we know when $(\tau_{h-1},o)$ is reachable,
\begin{align*}
&\BP(t|\tau_{h-1},o,a)=\BP(o(t)|\tau_{h-1},o,a;\doi(a(t)))=\frac{\BP(o,a,o(t)|\tau_{h-1};\doi(a(t)))}{\BP(o,a|\tau_{h-1};\doi(a(t)))}\\
&\qquad=\frac{\BP(o,a,o(t)|\tau_{h-1};\doi(a(t)))}{\BP(o,a|\tau_{h-1})}=\frac{\BP(o,o(t)|\tau_{h-1};\doi(a,a(t)))\BP(a|\tau_{h-1},o)}{\BP(o|\tau_{h-1})\BP(a|\tau_{h-1},o)}\\
&\qquad=\frac{\BP(o,o(t)|\tau_{h-1};\doi(a,a(t)))}{\BP(o|\tau_{h-1})}=\frac{m_{(o,a,t),h}^{\top}q_{\tau_{h-1}}}{m_{o,h}^{\top}q_{\tau_{h-1}}},
\end{align*}
where the last step comes from the definition \eqref{eq:psr def}.

Take $t$ to be the tests in $\UC_{h+1}$, then we have
\begin{align*}
q_{\tau_{h-1},o,a}=M_{o,a,h}q_{\tau_{h-1}}/m_{o,h}^{\top}q_{\tau_{h-1}}.
\end{align*}

\subsection{Proof of Lemma~\ref{lem:product}}
\label{proof:lem product}
We first prove \eqref{eq:product 2}. Notice that we have
\begin{align*}
& \BP^{\pi}(\tau_h) q_{\tau_h}  = [\BP(u|\tau_h)\BP^{\pi}(\tau_h)]_{u\in\UC_{h+1}}=(\BP(u|\tau_h))_{u\in\UC_{h+1}}\BP(o_h|\tau_{h-1})\pi(a_h|\tau_{h-1},o_h)\BP^{\pi}(\tau_{h-1})\\
&\quad=(\BP(o_h,o(u)|\tau_{h-1};\doi(a_h,a(u))))_{u\in\UC_{h+1}}\cdot\pi(a_h|\tau_{h-1},o_h)\BP^{\pi}(\tau_{h-1})\\
&\quad=M_{o_h,a_h,h}(q_{\tau_{h-1}}\BP^{\pi}(\tau_{h-1}))\pi(a_h|\tau_{h-1},o_h)\\
&\quad=\cdots=\prod_{l=1}^{h}M_{o_l,a_l,l}q_0\times\pi(\tau_h)=b_{\tau_h}\times\pi(\tau_h),
\end{align*}
where the third step comes from the definition \eqref{eq:psr def}. In particular, for any trajectory $\tau_H$, we have
\begin{align*}
\BP^{\pi}(\tau_H)=\pi(a_H|\tau_{H-1},o_H)(m_{o_H,H}^{\top}q_{\tau_{H-1}})\BP^{\pi}(\tau_{H-1})=m_{o_H,H}^{\top}\cdot\prod_{h=1}^{H-1}M_{o_h,a_h,h}\cdot q_0\cdot\pi(\tau_H).
\end{align*}
\section{$\eps$-Bracket Number of $\FC$}
\label{sec:bracket}
In this section we introduce some basic properties of the $\eps$-bracket number $\N_{\FC}(\eps)$. We first consider PSRs and then take POMDPs as special examples. 

\subsection{PSRs} 
For PSRs, let us define the covering number for the parameters $\{M_{o,a,h;f},q_{0;f}\}$ as follows:
\begin{definition}[$\eps$-covering number]
\label{def:covering m}
The $\eps$-covering number of $\FC$, denoted by $\Z_{\FC}(\eps)$, is the minimum integer $n$ such that there exists a function class $\FC'$ with $|\FC'|=n$ and for any $f\in\FC$ there exists $f'\in\FC'$ such that $\max_{o\in\BO, a\in\A,h\in[H-1],u\in\UC_{h+1}}\Vert m_{(o,a,u),h;f}-m_{(o,a,u),h;f'}\Vert_{\infty}\leq\eps$ and $\Vert q_{0;f}-q_{0;f'}\Vert_{\infty}\leq\eps$. 
\end{definition}
Here $\FC'$ does not need to be valid PSR model classes and $m_{(o,a,u),h;f'}$ does not need to belong to the column space of $K_{h;f'}$. That said, we still use $\BP^{\pi}_{f'}(\tau_H)$ to denote the product $m_{o_H,H}^{\top}\cdot\prod_{h=1}^{H-1}M_{o_h,a_h,h;f'}\cdot q_{0;f'}\cdot\pi(\tau_H)$ where $m_{o_H,H}=e_{o_H,H}$, although this might no longer be a valid distribution. Then the following lemma shows that the bracket number can be upper bounded by the covering number, whose proof is deferred to Appendix~\ref{proof:bracket}.
\begin{lemma}
\label{lem:braket}
Given $\FC$ and any $\eps>0$, suppose Assumption~\ref{ass:regular},\ref{ass:observe} and \ref{ass:function} hold, then we have
\begin{align*}
\N_{\FC}(\eps)\leq\Z_{\FC}(\alpha\eps/(8|\BO|^{H+1}|\A|^HH|\UA|^2|\UC|)).
\end{align*}
\end{lemma}

Since the log covering number $\log\Z_{\FC}(\eps)$ typically scales with $\log(1/\eps)$, Lemma~\ref{lem:braket} shows that $\log\N_{\FC}(\eps)$ also scales with polynomial $H$ in general. 

\paragraph{Tabular PSRs.} Let us consider the tabular cases for example where we directly use $\{M_{o,a,h},q_{0}\}_{o\in\BO,a\in\A,h\in[H-1]}$ as the parameters of $\FC$ and assume for all $f\in\FC$ we have
\begin{align*}
 \max_{o\in\BO, a\in\A,h\in[H-1],u\in\UC_{h+1}}\Vert m_{(o,a,u),h;f}\Vert_{\infty}\leq1,\Vert q_{0;f}\Vert_{\infty}\leq 1 
\end{align*}
without loss of generality. Then we know
\begin{align*}
\log\Z_{\FC}(\alpha\eps/(8|\BO|^{H+1}|\A|^HH|\UA^2||\UC|))\leq\BO(|\UC|^2|\BO||\A|H^2\log(H|\BO||\A||\UA||\UC|/(\alpha\eps)),
\end{align*}
which implies that
\begin{align*}
\log\N_{\FC}(\eps)\leq\BO(|\UC|^2|\BO||\A|H^2\log(H|\BO||\A||\UA||\UC|/(\alpha\eps))).
\end{align*}

\subsection{POMDPs} 
For POMDPs, we can obtain a more efficient function class by modeling the emission matrix $\MO_h$, transition kernel $\MT_h$ and initial state distribution $\mu_1$ instead of $\{M_{o,a,h;f},q_{0;f}\}$. Let us define the covering number for the parameters $\{\MT_{h;f},\MO_{h;f},\mu_{1;f}\}_{h\in[H]}$ as follows:
\begin{definition}
	The $\eps$-covering number of $\{\MT_{h;f},\MO_{h;f},\mu_{1;f}\}_{h\in[H],f\in\FC}$, denoted by $\VC_{\FC}(\eps)$, is the minimum integer $n$ such that there exists a function class $\FC'$ with $|\FC'|=n$ and for any $f\in\FC$ there exists $f'\in\FC'$ such that $\max_{h\in[H-1],a\in\A}\Vert \MT_{h,a;f}-\MT_{h,a;f'}\Vert_{\infty,\infty}\leq\eps, \max_{h\in[H]}\Vert \MO_{h;f}-\MO_{h;f'}\Vert_{\infty,\infty}\leq\eps$ and $\Vert \mu_{1;f}-\mu_{1;f'}\Vert_{\infty}\leq\eps$. 
\end{definition} 

Then we have the following lemma:
\begin{lemma}
	\label{lem:net 3}
	For any $f\in\FC$ and $0<\eps_1\leq 1$, suppose $f'$ satisfies $\max_{h\in[H-1],a\in\A}\Vert \MT_{h,a;f}-\MT_{h,a;f'}\Vert_{\infty,\infty}\leq\eo, \max_{h\in[H]}\Vert \MO_{h;f}-\MO_{h;f'}\Vert_{\infty,\infty}\leq\eo$ and $\Vert \mu_{1;f}-\mu_{1;f'}\Vert_{\infty}\leq\eo$, where 
	\begin{align*} 
		\eo=\eps_1/(14|\BO|^2). 
	\end{align*}
	Then for any policy $\pi$, we have
	\begin{align*}
		\sum_{\tau_H}|\BP_{f'}^{\pi}(\tau_H)-\BP_{f}^{\pi}(\tau_H)|\leq\eps_1.
	\end{align*}
\end{lemma} 
The proof is omitted here since it follows similar arguments in the proof of Lemma~\ref{lem:net 1}. Therefore, following the arguments in the proof of Lemma~\ref{lem:braket}, we know
\begin{align*}
	\N_{\FC}(\eps)\leq\VC_{\FC}(\eps/(28|\BO|^{H+2}|\A|^H)).
\end{align*}

\paragraph{Tabular POMDPs.} For tabular POMDPs where $\{\MT_{h;f},\MO_{h;f},\mu_{1;f}\}_{h\in[H]}$ is modeled directly, it can be observed that $\log\VC_{\FC}(\eps)=\BO(H|\BO||\SC|^2|\A|\log(1/\eps))$. Therefore we have
\begin{align*}
%	\label{eq:conv tab}
	\log\N_{\FC}(\eps)\leq \poly(|\BO|,|\A|,|\SC|,H,\log(1/\eps)).
\end{align*}

\paragraph{Low-rank POMDPs.} For low-rank POMDPs, when utilizing the function class introduced in Section~\ref{sec:example}, we can obtain that $\log\VC_{\FC}(\eps)\leq\log\nphi(\eps/(3\dtran))+\log\npsi(\eps/(3\dtran))+\log\nmo(\eps/3)+\log\nmu(\eps)$, which implies that
\begin{align*}
%	\label{eq:linear 2}
	\log\N_{\FC}(\eps)\leq \log\nphi(\epsl/\dtran)+\log\npsi(\epsl/\dtran)+\log\nmo(\epsl)+\log\nmu(\epsl),
\end{align*}
where $\epsl:=\BO\big(\eps/(|\BO|^{H+2}|\A|^H)\big)$.

\paragraph{Linear POMDPs.} For linear POMDPs, when utilizing the function class introduced in Section~\ref{sec:example}, it can be calculated that $\log\VC_{\FC}(\eps)=\BO(H\prod_{i=1}^5d_i\log(\prod_{i=1}^5d_i/\eps))$, which implies that
\begin{align*}
%	\label{eq:linear 2}
	\log\N_{\FC}(\eps)\leq \BO\big(H^2\prod_{i=1}^5d_i\log(|\BO||\A|/\eps)\big).
\end{align*}

\subsection{Proof of Lemma~\ref{lem:braket}}
%%%%%%%\masa{this comes before D.2.? Now, in D2, we already refer to Lemma 16.}\whz{Yeah. I just want to present the results first and defer the proofs.}
\label{proof:bracket}
First let us prove that $\BP^{\pi}_{f}(\cdot)$ is Lipschitz continuous with respect to $\{M_{o,a,h;f},q_{0;f}\}$ for any policy $\pi$, as shown in the following lemma:
\begin{lemma}
\label{lem:net 1}
For any $f\in\FC$ and $0<\eps_1\leq|\UA|$, suppose $f'$ satisfies
\begin{align*}
\max_{o\in\BO, a\in\A,h\in[H-1],u\in\UC_{h+1}}\Vert m_{(o,a,u),h;f}-m_{(o,a,u),h;f'}\Vert_{\infty}\leq\eo,\Vert q_{0;f}-q_{0;f'}\Vert_{\infty}\leq\eo,
\end{align*}
  where 
\begin{align*} 
\eo=\alpha\eps_1/(4H|\UA|^2|\UC||\BO|). 
\end{align*}
Then for any policy $\pi$, we have
\begin{align*}
\sum_{\tau_H}|\BP_{f'}^{\pi}(\tau_H)-\BP_{f}^{\pi}(\tau_H)|\leq\eps_1.
\end{align*}
\end{lemma} 

Now consider the minimum $\eo$-covering net of $\FC$, denoted by $\FC'$. Then by the definition of minimum covering net, we know for any $f\in\FC'$, there exists $f'\in\FC$ such that
\begin{align*}
\max_{o\in\BO, a\in\A,h\in[H-1],u\in\UC_{h+1}}\Vert m_{(o,a,u),h;f}-m_{(o,a,u),h;f'}\Vert_{\infty}\leq\eo,\Vert q_{0;f}-q_{0;f'}\Vert_{\infty}\leq\eo.
\end{align*}

Using Lemma~\ref{lem:net 1}, we know for any policy $\pi$ and trajectory $\tau_H$, 
\begin{align*}
\BP^{\pi}_{f'}(\tau_H)-\eps_1\leq\BP^{\pi}_f(\tau_H)\leq\BP^{\pi}_{f'}(\tau_H)+\eps_1.
\end{align*}

Therefore, let us define $g^{f'}_1(\pi,\cdot)=\BP^{\pi}_{f'}(\cdot)-\eps_1$ and $g^{f'}_{2}(\pi,\cdot)=\BP^{\pi}_{f'}(\cdot)+\eps_1$, then the set $\{[g^{f'}_1,g^{f'}_2]:f'\in\FC'\}$ is a $2\eps_1(|\BO||\A|)^H$-bracket of $\FC$ where we use the fact that there are at most $(| \BO | |\A|)^H$ many trajectories. Let $2\eps_1(|\BO||\A|)^H=\eps$ and then we have
\begin{align*}
\N_{\FC}(\eps)\leq\Z_{\FC}(\alpha\eps/(8|\BO|^{H+1}|\A|^HH|\UA|^2|\UC|)).
\end{align*}

\subsection{Proof of Lemma~\ref{lem:net 1}}
\label{proof:net 1}
We use Lemma~\ref{lem:key} to prove this lemma via induction. First notice that we have for any $o\in\BO,a\in\A,h\in[H-1],u\in\UC_{h+1}$,
\begin{align}
	\label{eq:net 1-1}
	\Vert m_{(o,a,u),h;f'}-m_{(o,a,u),h;f}\Vert_{\infty}\leq\eo,
\end{align} 
\begin{align}
	\label{eq:net 1-2}
	\Vert q_{0;f'}-q_{0;f}\Vert_{\infty}\leq\eo.
\end{align}

In the following discussion, we use $q_0,m_{(o,a,u),h},M_{o,a,h},b_{\tau_h}$ to denote $q_{0;f},m_{(o,a,u),h;f},\break M_{o,a,h;f},b_{\tau_h;f}$ and $q'_0,m'_{(o,a,u),h},M'_{o,a,h},b'_{\tau_h}$ to denote $q_{0;f'},m_{(o,a,u),h;f'},M_{o,a,h;f'},b_{\tau_h;f'}$ to simplify writing. Next we use induction to prove the lemma.

For the base case, we have $b_{\tau_0}=q_0,b'_{\tau_0}=q'_0$. Therefore from \eqref{eq:net 1-2} we have
\begin{align*}
	\Vert b_{\tau_0}-b'_{\tau_0}\Vert_1\leq|\UC|\eo\leq\eps_1.
\end{align*}

Now suppose for any $h'\leq h$ where $h\in[H-2]^+$ and policy $\pi$, we have $\sum_{\tau_{h'}}\Vert b_{\tau_{h'}}-b'_{\tau_{h'}}\Vert_1\times\pi(\tau_{h'})\leq\eps_1$. Notice that here $f'$ might not satisfy Assumption~\ref{ass:function}, but from the proof of Lemma~\ref{lem:performance} we can see that Lemma~\ref{lem:performance} still holds since $f\in\FC$. Therefore we have for any policy $\pi$,
\begin{align}
	&\sum_{\tau_{h+1}}\Vert b_{\tau_{h+1}}-b'_{\tau_{h+1}}\Vert_1\times\pi(\tau_{h+1})\notag\\
	&\qquad\leq\frac{|\UA|}{\alpha}\bigg(\sum_{l=1}^{h+1}\sum_{\tau_l}\Vert[M_{o_l,a_l,l}-M'_{o_l,a_l,l}]b'_{\tau_{l-1}}\Vert_1\times\pi(\tau_l)+\Vert q_0-q'_0\Vert_1\bigg).\label{eq:net 1-3}
\end{align}

From \eqref{eq:net 1-1}, we know for any $l\in[h+1]$,
\begin{align}
	&\sum_{\tau_l}\Vert[M_{o_l,a_l,l}-M'_{o_l,a_l,l}]b'_{\tau_{l-1}}\Vert_1\times\pi(\tau_l)\notag\\
	&\qquad\leq\eo|\UC|\sum_{\tau_l}\Vert b'_{\tau_{l-1}}\Vert_1\times\pi(\tau_l)\notag\\
	&\qquad=\eo|\UC||\BO|\sum_{\tau_{l-1}}\Vert b'_{\tau_{l-1}}\Vert_1\times\pi(\tau_{l-1})\notag\\
	&\qquad\leq\eo|\UC||\BO|\sum_{\tau_{l-1}}(\Vert b_{\tau_{l-1}}\Vert_1\times\pi(\tau_{l-1})+\Vert b_{\tau_{l-1}}-b'_{\tau_{l-1}}\Vert_1\times\pi(\tau_{l-1}))\notag\\
	&\qquad\leq\eo|\UC||\BO|\sum_{\tau_{l-1}}(|\UA|+\Vert b_{\tau_{l-1}}-b'_{\tau_{l-1}}\Vert_1\times\pi(\tau_{l-1}))\notag\\
	&\qquad\leq\eo|\UC||\BO|(\eps_1+|\UA|).\label{eq:net 1-4}
\end{align}
Here the first step comes from Cauchy-Schwartz inequality and \eqref{eq:net 1-1}. The fourth step comes from the fact that $(b_{\tau_{l-1}}\pi(\tau_{l-1}))_{u}=\BP_f(u|\tau_{l-1})\BP^{\pi}_f(\tau_{l-1})$ and thus $\sum_{\tau_{l-1}}\Vert b_{\tau_{l-1}}\Vert_1\times\pi(\tau_{l-1})=\sum_{\tau_{l-1}}\Vert q_{\tau_{l-1};f}\Vert_1\cdot\BP^{\pi}_f(\tau_{l-1})\leq|\UA|\sum_{\tau_{l-1}}\BP^{\pi}_f(\tau_{l-1})=|\UA|$. The last step comes from the induction hypothesis. 

Substituting \eqref{eq:net 1-2} and \eqref{eq:net 1-4} into \eqref{eq:net 1-3}, we have
\begin{align*}
	\sum_{\tau_{h+1}}\Vert b_{\tau_{h+1}}-b'_{\tau_{h+1}}\Vert_1\times\pi(\tau_{h+1})\leq\eps_1.
\end{align*}

Therefore, we have for all $h\in[H-1]$ and policy $\pi$,
\begin{align}
	\label{eq:net 1-5}
	\sum_{\tau_{h}}\Vert b_{\tau_{h}}-b'_{\tau_{h}}\Vert_1\times\pi(\tau_{h})\leq\eps_1.
\end{align}

Notice that from Lemma~\ref{lem:product} and Assumption~\ref{ass:observe} (where we let $m^k_{o_H,H}=m_{o_H,H}=e_{o_H,H}$), we have for any policy $\pi$,
\begin{align}
	\label{eq:net 1-6}
	\sum_{\tau_H}|P^{\pi}_{f'}(\tau_H)-P^{\pi}_{f}(\tau_H)|\leq\sum_{\tau_{H-1}}\Vert b_{\tau_{H-1}}-b'_{\tau_{H-1}}\Vert_1\times\pi(\tau_{H-1}).
\end{align}

Combining \eqref{eq:net 1-5} and \eqref{eq:net 1-6}, we have for all policy $\pi$
\begin{align*}
	\sum_{\tau_H}|P^{\pi}_{f'}(\tau_H)-P^{\pi}_{f}(\tau_H)|\leq\eps_1.
\end{align*}
This concludes our proof.
\section{Proof of Lemma~\ref{lem:redundancy}}
\label{proof:lem redundancy}
We first show that $b_{\tau_h;f}=b_{\tau_h;f'}$ for any $h\in[H-1]^+$ and trajectory $\tau_h$. We prove this via induction. For the base case where $h=0$, $b_{\tau_h;f}=b_{\tau_h;f'}=q_{0;f}$. Next for any $h\in[H-2]^+$, we suppose $b_{\tau_{h'};f}=b_{\tau_{h'};f'}$ for any $h'\in[h]^+$ and trajectory $\tau_{h'}$. Then for any trajectory $\tau_{h+1}$, let $\pi_{\tau_{h}}$ denote the policy that always takes the action sequence in $\tau_{h}$. From \eqref{eq:product 2} in Lemma~\ref{lem:product}, we have
\begin{align}
&b_{\tau_{h+1};f}=M_{o_{h+1},a_{h+1},h+1;f}b_{\tau_{h};f}=M_{o_{h+1},a_{h+1},h+1;f}b_{\tau_{h};f}\pi_{\tau_h}(\tau_h)\notag\\
&\qquad=M_{o_{h+1},a_{h+1},h+1;f}q_{\tau_{h};f}\BP^{\pi_{\tau_h}}(\tau_h)\notag\\
%&\qquad=M_{o_{h+1},a_{h+1},h+1;f}\bq_{\tau_{h};f}\BP^{\pi_{\tau_h}}(\tau_h)\notag\\
&\qquad=(m_{(o_{h+1},a_{h+1},u),h+1;f}^{\top}q_{\tau_h;f})_{u\in\UC_{h+2}}\BP^{\pi_{\tau_h}}(\tau_h).\label{eq:redun 1}
\end{align}

Similarly, since $b_{\tau_h;f'}=b_{\tau_h,f}$, we have
\begin{align}
b_{\tau_{h+1};f'}=(m_{(o_{h+1},a_{h+1},u),h+1;f'}^{\top}q_{\tau_h;f})_{u\in\UC_{h+2}}\BP^{\pi_{\tau_h}}(\tau_h).\label{eq:redun 2}
\end{align}

From \eqref{eq:core history}, we know $q_{\tau_h;f}$ belongs to the column space of $K_{h;f}$. This implies that for any $u\in\UC_{h+2}$
\begin{align}
m_{(o_{h+1},a_{h+1},u),h+1;f}^{\top}q_{\tau_h;f}=m_{(o_{h+1},a_{h+1},u),h+1;f'}^{\top}q_{\tau_h;f}.\label{eq:redun 3}
\end{align}

Combining \eqref{eq:redun 1},\eqref{eq:redun 2} and \eqref{eq:redun 3}, we have
\begin{align*}
b_{\tau_{h+1};f}=b_{\tau_{h+1};f'}.
\end{align*}

Therefore, for any $h\in[H-1]^+$ and trajectory $\tau_h$, we have
\begin{align*}
b_{\tau_{h};f}=b_{\tau_{h};f'}.
\end{align*}

This suggests that for any policy $\pi$ and trajctory $\tau_{H-1}$, we have
\begin{align*}
(\BP_f(u|\tau_{H-1})\BP^{\pi}_f(\tau_{H-1}))_{u\in\UC_{H}}=(\BP_{f'}(u|\tau_{H-1})\BP^{\pi}_{f'}(\tau_{H-1}))_{u\in\UC_H}.
\end{align*}
Therefore with Assumption~\ref{ass:observe} we have for any policy $\pi$ and trajectory $\tau_{H}$, we have
\begin{align*}
\BP^{\pi}_f(\tau_H)=\BP^{\pi}_{f'}(\tau_{H}).
\end{align*}
\section{Proofs of Lemmas in Section~\ref{sec:sketch}}
\subsection{Proof of Lemma~\ref{lem:optimism}}
\label{proof:optimism}
To prove $f^*\in\B^k$, we need to show that $\sum_{(\pi,\tau_H)\in\D}\log\BP^{\pi}_{f^*}(\tau_H)$ is large. To simplify writing, we denote the $(\pi,\tau_H)$ pairs in $\D$ at the end of $K$-th iteration by $\{(\pi^{i},\tau^{i}_H)\}_{i=1}^{n_K}$, which are indexed by their collection order. Notice that $n_K\leq KH|\UA|$. To deal with potentially infinite function clas $\FC$, we first consider its minimum $\eb$-bracket net $\GC$ where $\eb=1/(KH|\UA|)$ and the set of all upper bound functions in $\GC$, i.e., $\GCU:=\{f':\exists f,\text{ such that }[f,f']\in\GC\}$. Then we are able to bound the difference bewteen $\sum_{(\pi,\tau_H)\in\D}\log\BP^{\pi}_{f^*}(\tau_H)$ and $\sum_{(\pi,\tau_H)\in\D}\log\BP^{\pi}_{f}(\tau_H)$ for any $f\in\FC$ via Cram\'er-Chernoff's method as in \citep{liu2022partially}.

Fix any $f'\in\GCU,t\in[n_K]$ and let $\Fil_t$ denote the filtration induced by $\{(\pi^i,\tau^i)\}_{i=1}^{t-1}\cup\{\pi^t\}$. We have:
\begin{align*}
	&\E\bigg[\exp\bigg(\sum_{i=1}^t\log\bigg(\frac{\BP^{\pi^i}_{f'}(\tau_H^{i})}{\BP^{\pi^i}_{f^*}(\tau_H^{i})}\bigg)\bigg)\bigg]\\
	&\qquad=\E\bigg[\exp\bigg(\sum_{i=1}^{t-1}\log\bigg(\frac{\BP^{\pi^i}_{f'}(\tau_H^{i})}{\BP^{\pi^i}_{f^*}(\tau_H^{i})}\bigg)\bigg)\cdot\E\bigg[\exp\bigg(\log\bigg(\frac{\BP^{\pi^t}_{f'}(\tau_H^{t})}{\BP^{\pi^t}_{f^*}(\tau_H^{t})}\bigg)\bigg)\bigg|\Fil_t\bigg]\bigg]\\
	&\qquad=\E\bigg[\exp\bigg(\sum_{i=1}^{t-1}\log\bigg(\frac{\BP^{\pi^i}_{f'}(\tau_H^{i})}{\BP^{\pi^i}_{f^*}(\tau_H^{i})}\bigg)\bigg)\cdot\E\bigg[\frac{\BP^{\pi^t}_{f'}(\tau_H^{t})}{\BP^{\pi^t}_{f^*}(\tau_H^{t})}\bigg|\Fil_t\bigg]\bigg]\\
	&\qquad=\E\bigg[\exp\bigg(\sum_{i=1}^{t-1}\log\bigg(\frac{\BP^{\pi^i}_{f'}(\tau_H^{i})}{\BP^{\pi^i}_{f^*}(\tau_H^{i})}\bigg)\bigg)\cdot\sum_{\tau_H}P^{\pi^t}_{f'}(\tau_H)\bigg]\\
	&\qquad\leq\E\bigg[\exp\bigg(\sum_{i=1}^{t-1}\log\bigg(\frac{\BP^{\pi^i}_{f'}(\tau_H^{i})}{\BP^{\pi^i}_{f^*}(\tau_H^{i})}\bigg)\bigg)\cdot\bigg(1+\frac{1}{KH|\UA|}\bigg)\bigg],
\end{align*}
where the last step is due to the fact that $\GC$ is the minimum $\eb$-bracket net, which implies that there exists $f\in\FC$ such that $\Vert\BP^{\pi}_f(\cdot)-\BP^{\pi}_{f'}(\cdot)\Vert_1\leq\eb$ for any policy $\pi$ and thus $\Vert\BP^{\pi}_{f'}(\cdot)\Vert_1\leq1+\eb$. Repeat the above arguments and we have
\begin{align*}
	\E\bigg[\exp\bigg(\sum_{i=1}^t\log\bigg(\frac{\BP^{\pi^i}_{f'}(\tau_H^{i})}{\BP^{\pi^i}_{f^*}(\tau_H^{i})}\bigg)\bigg)\bigg]\leq e.
\end{align*}

Then by Markov's inequality we have for any $\delta\in(0,1]$,
\begin{align*}
	&\BP\bigg(\sum_{i=1}^t\log\bigg(\frac{\BP^{\pi^i}_{f'}(\tau_H^{i})}{\BP^{\pi^i}_{f^*}(\tau_H^{i})}\bigg)>\log(1/\delta)\bigg)\\
	&\qquad\leq\E\bigg[\exp\bigg(\sum_{i=1}^t\log\bigg(\frac{\BP^{\pi^i}_{f'}(\tau_H^{i})}{\BP^{\pi^i}_{f^*}(\tau_H^{i})}\bigg)\bigg)\bigg]\cdot\exp[-\log(1/\delta)]\leq e\delta.
\end{align*} 

Therefore by union bound, we have for all $f'\in\GCU, t\in[n_K]$,
\begin{align*}
	&\BP\bigg(\sum_{i=1}^t\log\bigg(\frac{\BP^{\pi^i}_{f'}(\tau_H^{i})}{\BP^{\pi^i}_{f^*}(\tau_H^{i})}\bigg)>c\log(\ncov KH|\UA|/\delta)\bigg)\leq \delta/2,
\end{align*} 
where $c$ is a universal constant.

Finally, due to the definition of $\eps$-bracket net, we know for all $f\in\FC$, there exists $f'\in\GCU$ such that $\BP^{\pi}_f(\tau_H)\leq\BP^{\pi}_{f'}(\tau_H)$ for any trajectory $\tau_H$ and policy $\pi$. Therefore we have for all $f\in\FC, t\in[n_K]$,
\begin{align*}
	&\BP\bigg(\sum_{i=1}^t\log\bigg(\frac{\BP^{\pi^i}_{f}(\tau_H^{i})}{\BP^{\pi^i}_{f^*}(\tau_H^{i})}\bigg)>c\log(\ncov KH|\UA|/\delta)\bigg)\leq \delta/2,
\end{align*}
which implies that $f^*\in\B^k$ for all $k\in[K]$ with probability at least $1-\delta/2$. This concludes our proof.

\subsection{Proof of Lemma~\ref{lem:performance}}
\label{proof:performance}
First, notice that we can decompose the left hand side of \eqref{eq:performance 1} into the following sequence of terms via triangle inequality:
\begin{align}
&\sum_{\tau_h}\bigg\Vert\prod_{l=1}^{h}M^k_{o_l,a_l,l}\cdot q^k_0-\prod_{l=1}^{h}M_{o_l,a_l,l}\cdot q_0\bigg\Vert_1\times\pi(\tau_h)\notag\\
&\qquad\leq\sum_{j=1}^h\sum_{\tau_h}\bigg\Vert\prod_{l=j+1}^{h}M^k_{o_l,a_l,l}\bigg(M^k_{o_j,a_j,j}-M_{o_j,a_j,j}\bigg)\cdot b_{\tau_{j-1}}\bigg\Vert_1\times\pi(\tau_h)\notag\\
&\qquad\quad+\sum_{\tau_h}\bigg\Vert\prod_{l=1}^{h}M^k_{o_l,a_l,l}\bigg(q^k_0-q_0\bigg)\bigg\Vert_1\times\pi(\tau_h).\label{eq:performance 2}
\end{align}

Then fix $j\in[h]$ and consider the term $\sum_{\tau_h}\big\Vert\prod_{l=j+1}^{h}M^k_{o_l,a_l,l}\big(M^k_{o_j,a_j,j}-M_{o_j,a_j,j}\big)\cdot b_{\tau_{j-1}}\big\Vert_1\times\pi(\tau_h)$ in \eqref{eq:performance 2}. We have
\begin{align}
&\sum_{\tau_h}\bigg\Vert\prod_{l=j+1}^{h}M^k_{o_l,a_l,l}\bigg(M^k_{o_j,a_j,j}-M_{o_j,a_j,j}\bigg)\cdot b_{\tau_{j-1}}\bigg\Vert_1\times\pi(\tau_h)\notag\\
&\qquad=\sum_{\tau_{j}}\pi(\tau_j)\sum_{\tau_{j+1:h}}\bigg\Vert\prod_{l=j+1}^{h}M^k_{o_l,a_l,l}\bigg(M^k_{o_j,a_j,j}-M_{o_j,a_j,j}\bigg)\cdot b_{\tau_{j-1}}\bigg\Vert_1\times\pi(\tau_{j+1:h}|\tau_j)\notag\\
&\qquad\leq\frac{|\UA|}{\alpha}\sum_{\tau_{j}}\bigg\Vert\bigg(M^k_{o_j,a_j,j}-M_{o_j,a_j,j}\bigg)\cdot b_{\tau_{j-1}}\bigg\Vert_1\cdot\pi(\tau_j),\label{eq:performance 3}
\end{align}
where the last step comes from Lemma~\ref{lem:key}.

Similarly, apply Lemma~\ref{lem:key} to the second part of \eqref{eq:performance 2} and we have
\begin{align}
\sum_{\tau_h}\bigg\Vert\prod_{l=1}^{h}M^k_{o_l,a_l,l}\bigg(q^k_0-q_0\bigg)\bigg\Vert_1\times\pi(\tau_h)\leq\frac{|\UA|}{\alpha}\Vert q^k_0-q_0\Vert_1.\label{eq:performance 4}
\end{align}

Substituting \eqref{eq:performance 3} and \eqref{eq:performance 4} into \eqref{eq:performance 2}, we can obtain
\begin{align*}
&\sum_{\tau_h}\bigg\Vert\prod_{l=1}^{h}M^k_{o_l,a_l,l}\cdot q^k_0-\prod_{l=1}^{h}M_{o_l,a_l,l}\cdot q_0\bigg\Vert_1\times\pi(\tau_h)\notag\\
&\qquad\leq\frac{|\UA|}{\alpha}\bigg(\sum_{l=1}^{h}\sum_{\tau_l}\Vert[M^k_{o_l,a_l,l}-M_{o_l,a_l,l}]b_{\tau_{l-1}}\Vert_1\times\pi(\tau_l)+\Vert q^k_0-q_0\Vert_1\bigg).
\end{align*}
This concludes our proof.
\subsection{Proof of Lemma~\ref{lem:key}}
\label{proof:key}
First, based on Lemma~\ref{lem:redundancy}, we have chosen $m_{(o,a,u),j_1;f}$ to belong to the column space of $K_{j_1-1;f}$, which implies that
\begin{align*}
&\sum_{\tau_{j_1:j_2}}\bigg\Vert\prod_{j=j_1}^{j_2}M_{o_j,a_j,j;f}x\bigg\Vert_1\times\pi(\tau_{j_1:j_2}|\tau_{j_1-1})\\
&\qquad=\sum_{\tau_{j_1:j_2}}\bigg\Vert\bigg(\prod_{j=j_1}^{j_2}M_{o_j,a_j,j;f}K_{j_1-1;f}\bigg)\bigg(K_{j_1-1;f}^{\dagger}x\bigg)\bigg\Vert_1\times\pi(\tau_{j_1:j_2}|\tau_{j_1-1}).
\end{align*}

Note that since the $l$-th column of $K_{j_1-1;f}$ is $q_{\tau_{j_1-1;f}^l;f}$, the $l$-th core history at step $j_1-1$ under the model induced by $f$, we have for any $l\in[\dhf]$,
%\begin{align*}
%&\sum_{\tau_{j_1:j_2}}\bigg\Vert\bigg(\prod_{j=j_1}^{j_2}M_{o_j,a_j,j;f}K_{j_1-1;f}e_l\bigg)\bigg\Vert_1\times\pi(\tau_{j_1:j_2}|\tau_{j_1-1})\\
%&\qquad=\sum_{\tau_{j_1:j_2}}\sum_{u\in\UC_{j_2+1}}\BP_f(u|(\tau^l_{j_1-1},\tau_{j_1:j_2}))\BP^{\pi_{\tau_{j_1-1}}}_f(\tau_{j_1:j_2}|\tau^l_{j_1-1;f})\\
%&\qquad=\sum_{\tau_{j_1:j_2}}\BP^{\pi_{\tau_{j_1-1}}}_f(\tau_{j_1:j_2}|\tau^l_{j_1-1;f})\bigg(\sum_{u\in\UC_{j_2+1}}\BP_f(u|(\tau^l_{j_1-1},\tau_{j_1:j_2}))\bigg)\\
%&\qquad\leq|\UA|\sum_{\tau_{j_1:j_2}}\BP^{\pi_{\tau_{j_1-1}}}_f(\tau_{j_1:j_2}|\tau^l_{j_1-1;f})=|\UA|.
%\end{align*}

\begin{align*}
&\sum_{\tau_{j_1:j_2}}\bigg\Vert\bigg(\prod_{j=j_1}^{j_2}M_{o_j,a_j,j;f}K_{j_1-1;f}e_l\bigg)\bigg\Vert_1\times \pi(\tau_{j_1:j_2}|\tau_{j_1-1})\\
&=\sum_{o_{j_1:j_2}}\sum_{a_{j_1:j_2}}\bigg\Vert\bigg(\prod_{j=j_1}^{j_2}M_{o_j,a_j,j;f}K_{j_1-1;f}e_l\bigg)\bigg\Vert_1\pi((o_{j_1:j_2},a_{j_1:j_2})|\tau_{j_1-1})\\
&=\sum_{o_{j_1:j_2}}\sum_{a_{j_1:j_2}}\sum_{u\in\UC_{j_2+1}}\BP_f(u|(\tau^l_{j_1-1},o_{j_1:j_2},a_{j_1:j_2}))\BP_f(o_{j_1:j_2}|\tau^l_{j_1-1;f}; \doi(a_{j_1:j_2-1})) \pi((o_{j_1:j_2},a_{j_1:j_2})|\tau_{j_1-1}) \\
&=\sum_{o_{j_1:j_2}}\sum_{a_{j_1:j_2}}\left \{\sum_{u\in\UC_{j_2+1}}\BP_f(u|(\tau^l_{j_1-1},o_{j_1:j_2},a_{j_1:j_2}))\right \}
\BP_f(o_{j_1:j_2}|\tau^l_{j_1-1;f}; \doi(a_{j_1:j_2-1})) \pi((o_{j_1:j_2},a_{j_1:j_2})|\tau_{j_1-1})\\
&\leq |\UA|\sum_{o_{j_1:j_2}}\sum_{a_{j_1:j_2}}
\BP_f(o_{j_1:j_2}|\tau^l_{j_1-1;f}; \doi(a_{j_1:j_2-1}))\pi((o_{j_1:j_2},a_{j_1:j_2})|\tau_{j_1-1}) \\
&\leq |\UA|. 
\end{align*}

Here in the second step $\pi((o_{j_1:j_2},a_{j_1:j_2})|\tau_{j_1-1})$ denotes $\prod_{j=j_1}^{j_2}\pi(a_j|\tau_{j_1-1},o_{j_1:j},a_{j_1:j-1})$ and the third step comes from Lemma~\ref{lem:product}. 

Therefore we have
\begin{align*}
&\sum_{\tau_{j_1:j_2}}\bigg\Vert\bigg(\prod_{j=j_1}^{j_2}M_{o_j,a_j,j;f}K_{j_1-1;f}\bigg)\bigg(K_{j_1-1;f}^{\dagger}x\bigg)\bigg\Vert_1\times\pi(\tau_{j_1:j_2}|\tau_{j_1-1})\\
&\qquad\leq|\UA|\bigg\Vert K_{j_1-1;f}^{\dagger}x\bigg\Vert_1\leq\frac{|\UA|}{\alpha}\Vert x\Vert_1,
\end{align*} 
where the third step comes from Assumption~\ref{ass:function}. This concludes our proof.

\section{Proof of Theorem~\ref{thm:lower bound}}
\label{proof:thm lower bound}
In this section we leverage the hard instance constructed in \cite{liu2022partially} to prove the lower bound, which is based on combinatorial lock. More specifically, we define a POMDP as follows:
\begin{itemize}
\item \textbf{State space:} There are two states, $\SC=\{s_g,s_b\}$.

\item \textbf{Observation space and emission matrices:} There are three observations, $\BO=\{o_g,o_b,\od\}$. For $h\in[H-1]$, we define the emission matrix as follows:
\begin{align*}
\MO_{h}=
\begin{pmatrix}
\sqrt{2}\alpha & 0\\
0 & \sqrt{2}\alpha\\
1-\sqrt{2}\alpha & 1-\sqrt{2}\alpha
\end{pmatrix}.
\end{align*} 
For $h=H$, we have
\begin{align*}
	\MO_{H}=
	\begin{pmatrix}
		1 & 0\\
		0 & 1\\
		0 & 0
	\end{pmatrix}.
\end{align*} 
This means that with probability $\alpha$ we can observe the current state and with probability $1-\alpha$ we only receive a dummy observation at step $h\in[H-1]$. At step $H$, though, we are able to observe the current state.  

\item\textbf{Action space and transition kernels:} There are $|\A|$ actions and the initial state is fixed as $s_g$. For each step $h\in[H-1]$, there exists a good action $a_{g,h}\in\A$ which is chosen uniformly at random from $\A$ such that if the agent is currently in $s_h=s_g$ and takes $a_{g,h}$, it will stay in $s_g$, i.e., $s_{h+1}=s_g$. Otherwise, the agent will always go to $s_{h+1}=s_b$.  

\item\textbf{Reward:} We define $r_{h}(o)=0$ for all $h\in[H-1]$ and $o\in\BO$. At step $H$, $r_{H}(o_g)=1$ while $r_{H}(o_b)=0$. This indicates that the agent will receive reward $1$ iff the agent takes $a_{g,h}$ along its way.

\end{itemize}

Since this POMDP satisifes weakly-revealing condition, we know $\BO$ is its core test set. Next we show that this POMDP satisfies Assumption~\ref{ass:regular}. First it is can be observed that $K_0=q_0=(\sqrt{2}\alpha,0,1-\sqrt{2}\alpha)^{\top}$ and we can verify that 
\begin{align*}
\Vert K_0^{\dagger}\Vert_{1\mapsto1}\leq1/\alpha.
\end{align*}

Then for any $h\in[H-1]$ and reachable history $\tau_h$, if $a_{1:h}=a_{g,1:h}$, we have
\begin{align*}
\BP(s_{h+1}=s_g|\tau_h)=1,\BP(s_{h+1}=s_b|\tau_h)=0,
\end{align*}
which implies that
\begin{align*}
\BP(o_{h+1}=o_g|\tau_h)=\sqrt{2}\alpha,\BP(o_{h+1}=o_b|\tau_h)=0,\BP(o_{h+1}=\od|\tau_h)=1-\sqrt{2}\alpha.
\end{align*}

Otherwise, if there exists $h'\in[h]$ such that $a_{h'}\neq a_{g,h'}$, then we have
\begin{align*}
\BP(s_{h+1}=s_b|\tau_h)=1,\BP(s_{h+1}=s_g|\tau_h)=0,
\end{align*}
which implies
\begin{align*}
	\BP(o_{h+1}=o_b|\tau_h)=\sqrt{2}\alpha,\BP(o_{h+1}=o_g|\tau_h)=0,\BP(o_{h+1}=\od|\tau_h)=1-\sqrt{2}\alpha.
\end{align*}

This suggests that $K_h=\MO_{h+1}$ for $h\in[H-1]$. On the otherhand, since $\sigma_{\min}(K_h)=\sigma_{\min}(\MO_{h+1})\geq\sqrt{2}\alpha$, we have for $h\in[H-1]$,
\begin{align*}
\Vert K_h^{\dagger}\Vert_{1\mapsto1}\leq\sqrt{2}\Vert K_h^{\dagger}\Vert_{2\mapsto2}\leq\sqrt{2}/(\sqrt{2}\alpha)=1/\alpha.
\end{align*}
This shows that the constructed POMDP satisfies Assumption~\ref{ass:regular}.

Now we only need to show that the constructed POMDP attains the lower bound in Theorem~\ref{thm:lower bound}. This has been proved in \cite{liu2022partially} and we include the proof here for completeness.

Suppose we can only interact with the POMDP for $T\leq\lfloor\frac{1}{2\sqrt{2}\alpha H}\rfloor$ episodes. Then we know the probability that both $s_g$ and $s_b$ only emit $\od$ in the first $H-1$ steps for all $T$ episodes is lower bounded by $(1-\sqrt{2}\alpha)^{1/(\sqrt{2}\alpha)}$ since $2\cdot\lfloor\frac{1}{2\sqrt{2}\alpha H}\rfloor\cdot(H-1)\leq1/(\sqrt{2}\alpha)$.

Now conditioned on the event that both $s_g$ and $s_b$ only emit $\od$ in the first $H-1$ steps for all $T$ episodes, we can only random guess the optimal action sequence $a_{g,1:H-1}$. Then if $T\leq |\A|^{H-1}/10$, the probability that we fail to guess the optimal action sequence is
\begin{align*}
\binom{|\A|^{H-1}-1}{K}\bigg/\binom{|\A|^{H-1}}{K}\geq 0.9,
\end{align*}

Therefore, with probability $0.9\times(1-\sqrt{2}\alpha)^{1/(\sqrt{2}\alpha)}\geq 1/6$, the agent can only learn that the action sequences it chooses in these $T$ episodes is incorrect, which implies that the agent can only random guess from the remained action sequences. Therefore, if $T\leq |\A|^{H-1}/10$, the policy that the agent outputs will be worse than $1/2$-optimal, which concludes our proof. 
\end{document}